\relax
\documentclass[letterpaper]{article}
\usepackage{times} 
\usepackage{helvet} 
\usepackage{courier} 
\usepackage[hyphens]{url} 
\usepackage{graphicx} 
\usepackage{graphicx} 
\usepackage{natbib}
\usepackage{caption}
\usepackage[switch]{lineno}
\usepackage{authblk}

\graphicspath{ {./Histograms/} }
\title{Streaming Algorithms for Stochastic Multi-armed Bandits}
\author[1]{Arnab Maiti}
\author[2]{Vishakha Patil}
\author[2]{Arindam Khan}
\affil[1]{Department of Computer Science and Engineering, Indian Institute of Technology, Kharagpur, India, \texttt{maitiarnab9@gmail.com}}
\affil[2]{Department of Computer Science and Automation, Indian Institute of Science, Bangalore, India, 
	 \texttt{patilv@iisc.ac.in}, \texttt{arindamkhan@iisc.ac.in}}
\date{}

\usepackage{color}
\usepackage{amsthm}
\usepackage{amsmath}
\usepackage{amsfonts}
\usepackage{amssymb}
\usepackage[dvipsnames]{xcolor}
\usepackage{algorithmic}
\usepackage{algorithm}
\usepackage{comment}
\usepackage{thmtools,thm-restate}
\usepackage{multirow}
\usepackage{subcaption}
\newtheorem{thm}{Theorem}
	\newtheorem{lem}[thm]{Lemma}
	\newtheorem{cor}[thm]{Corollary}

	\newtheorem{obs}{Observation}
	\newtheorem{definition}{Definition}

\definecolor{olive}{RGB}{50,190,30}
\usepackage[margin=1.2in]{geometry}

\newcommand{\MAB}{\textsc{MAB} }
\newcommand{\mab}{\textsc{MAB}}
\newcommand{\KL}{\mathtt{KL}}
\newcommand{\omegsamp}{\mathtt{\Omega}}
\newcommand{\ilog}{\mathtt{ilog}}
\newcommand{\arm}{\mathtt{arm}}
\newcommand{\mean}{\mu}

	\ifdefined\DEBUG
	\newcommand{\ari}[1]{\textcolor{Red}{#1}}
	
	 \newcommand{\vis}[1]{\textcolor{OliveGreen}{#1}}
	 \newcommand{\arn}[1]{\textcolor{NavyBlue}{#1}}
	  \def\rem#1{{\marginpar{\raggedright\scriptsize #1}}}
	  \newcommand{\arir}[1]{\rem{\textcolor{Red}{$\bullet$ #1}}}
	  
	  \newcommand{\vish}[1]{\rem{\textcolor{OliveGreen}{$\bullet$ #1}}}
	  \newcommand{\arnr}[1]{\rem{\textcolor{NavyBlue}{$\bullet$ #1}}}
	\else
	  \newcommand{\ari}[1]{#1}
	  \newcommand{\vis}[1]{#1}
	  \newcommand{\arn}[1]{#1}
	  
	  \newcommand{\arir}[1]{}
	  \newcommand{\vish}[1]{}
	  \newcommand{\arnr}[1]{}
	 
	\fi

\newcommand\ceil[1]{\lceil#1\rceil}
\begin{document}
\maketitle

\begin{abstract}
We study the Stochastic Multi-armed Bandit problem under bounded arm-memory. In this setting, the arms arrive in a stream, and the number of arms that can be stored in the memory at any time, is bounded. The decision-maker can only pull arms that are present in the memory.  
We address the problem from the perspective of two standard objectives: 1) regret minimization, and 2) best-arm identification. 

For   {\em regret minimization}, we settle an important open question by showing an almost tight hardness. We show $\Omega(T^{2/3})$ cumulative regret in expectation for arm-memory size of $(n-1)$, where $n$ is the number of arms. 

For {\em best-arm identification}, we study two algorithms.  
First, we present an $O(r)$ arm-memory $r$-round adaptive streaming algorithm to find an $\varepsilon$-best arm. 
In $r$-round adaptive streaming algorithm for  best-arm identification,  the arm pulls in each round are decided based on the observed outcomes in the earlier rounds. The best-arm is the output at the end of $r$ rounds. 
The upper bound on the sample complexity of our algorithm matches with the lower bound for any $r$-round adaptive streaming algorithm. Secondly, we present a heuristic to find the $\varepsilon$-best arm with optimal sample complexity, by storing only one extra arm in the memory.

 \end{abstract}

\section{Introduction}
\label{section:introduction}
The Stochastic Multi-armed Bandits (\mab) problem is a classical framework used to capture decision-making in uncertain environments. Starting with the seminal work of \cite{ROBBINS1952}, a significant body of work has been developed to address theoretical as well as practical aspects of the \MAB problem. See, e.g.  \cite{BUBECK2012} for a textbook treatment of the area. In addition to being theoretically interesting, the \MAB problem also finds many practical applications in multiple areas, including on-line advertising \citep{TRAN2014a}, crowd-sourcing \citep{TRAN2014b}, and clinical trials \citep{CHAKRAVORTY2014}. Hence, the study of \MAB and its variants is of central interest in multiple fields, including online learning and reinforcement learning.

In the \MAB setting, a decision-maker is faced with $n$ choices (called \emph{arms}) and has to sequentially choose one of the $n$ arms (referred to as \emph{pulling} an arm). Based on the pulled arm, the decision-maker gets a reward drawn from a corresponding reward distribution that is unknown to the decision-maker. The \MAB problem has been extensively studied with one of the following two goals: \emph{regret minimization} and \emph{best-arm identification}. 
In the regret minimization literature, several algorithms such as UCB1 \citep{AUER2002}, Thompson Sampling \citep{THOMPSON1933,AGRAWAL12}, and KL-UCB \citep{Garivier2011} have been proposed whose regret bounds are within a constant factor of the optimal regret \citep{LAI1985}. Algorithms for the best-arm identification problem such as the Median-Elimination algorithm by \cite{EVENDAR2002} have optimal (upto constants) sample complexity for this problem. 

Each of the algorithms mentioned above needs to store the reward statistics (e. g. the number of pulls and mean reward of an arm observed so far) of \emph{all} the arms in the memory. In many of the applications of the \MAB problem, the number of arms (set of advertisements, crowd-workers, etc.) could be very large and the algorithm may not be able to simultaneously store all the arms in the memory. Additionally, the arms could arrive online, i.e., the algorithm may not have access to the entire set of arms at the beginning. The streaming model, first formalized in the seminal work of \cite{ALON1996}, has been developed to handle data streams where the data arrives online and an algorithm has access to only a limited amount of memory.  In this work, we study a setting where an algorithm can store statistics from only a fixed number of arms $m < n$, where $m$ is called the space complexity of the algorithm. Now, our revised goal is to study the trade-off between \emph{space complexity} vs. \emph{expected regret} and \emph{space complexity} vs. \emph{sample complexity}, respectively.


We follow a standard model in this setup and  address both regret minimization and best-arm identification under streaming constraints. Here, $n$ arms arrive one by one in a stream, and we have bounded arm-memory of $m$, i.e., at any time step at most $m$ ($<n$) number of arms can be stored in the memory. We call $m$ the space complexity of the algorithm. 
At any time step $t$, the algorithm can only pull an arm that is currently in memory and then, if needed, the algorithm can choose to discard some of the arms that are currently in memory. If, at some time step $t$,  the number of arms in memory is less than $m$, then the algorithm can choose to store upto total $m$ arms in the memory. Note that, if the arm that is read into memory had previously been in memory and was subsequently discarded before being read back into memory, then the algorithm does not have access to any of the previous reward statistics of the arm. In streaming terminology, algorithms that are allowed to read back an arm that was previously discarded from memory are called \emph{multi-pass} algorithms. Otherwise, they are called \emph{single-pass} algorithms.

In one of the earliest works that studies the \MAB problem with bounded arm-memory, \cite{HERSCHKORN1996} study the infinite-armed \MAB problem where only a single arm is stored in the memory at any time step and an arm that is once discarded from the memory cannot be recalled. The work of \cite{Lu2011} studies a different variation of bounded memory in the learning from experts setting, where the constraint is on the number of time steps for which the rewards can be stored by the algorithm  and not on the number of arms. In our work, we focus on the finite-armed \MAB problem. Recently, the \MAB problem where arms arrive in a stream and arm-memory is bounded, has been studied in both regret minimization and best-arm identification frameworks.

The work of \cite{LIAU2018} and \cite{CHAUDHURI2019} studies the \MAB problem with bounded arm-memory to minimize the expected cumulative regret over $T$ time steps. The algorithms in both these works are multi-pass algorithms, i.e., they assume that an arm discarded from the memory can be again read back into the memory later. \cite{LIAU2018} propose an upper-confidence bound based algorithm with $O(1)$ space complexity, which achieves an expected cumulative regret bound of $O\big(\sum_{i\neq i^*} \log(\frac{\Delta_i}{\Delta})\frac{\log T}{\Delta_i}\big)$, which is within $\log(\frac{\Delta_i}{\Delta})$ factor of the UCB-1 regret bound (\cite{AUER2002}). Depending on the instance, the regret of this algorithm can be very high. \cite{CHAUDHURI2019} propose an algorithmic framework, which is given $m$ as input and uses a \MAB algorithm as a black-box. When the \MAB algorithm used is UCB-1, their algorithm achieves expected regret of $O(nm + n^{3/2}m\sqrt{T \log(T/nm)})$.

The recent work of \cite{ASSADI2020} studies the best-arm identification variant of this problem. First, they propose an algorithm for the best-coin identification problem (equivalent to \MAB with Bernoulli reward distributions), which keeps exactly one extra coin in the memory at any time and has optimal sample complexity. They further extend their algorithm to the top-$k$ coins identification problem, which stores $k$ coins in the memory and has optimal sample complexity. Crucially though, both algorithms assume that the gap parameter $\Delta$, which is the difference in the expected rewards of the best and the second-best coins, is known to the algorithm. Throughout our work, we deal with the case when $\Delta$ is not known to the algorithm.

\subsection{Our Contribution}
We study the \MAB problem under {\em bounded arm-memory}, where $n$ arms arrive in a stream and at most $m < n$ arms can be stored in the memory at any time. In this work, we study the trade-off between space complexity vs. expected regret and space complexity vs. sample complexity. \\

\noindent{\bf Regret minimization:}
Our first result settles an open question stated in both (\cite{LIAU2018} and \cite{CHAUDHURI2019}) pertaining to the lower bound on the expected cumulative regret in this model. Using information-theoretic machinery related to $\KL$-divergence, we show that any single-pass algorithm in this model will incur an expected regret of $\Omega\big({T^{2/3}}/{m^{7/3}}\big)$. Interestingly, this result holds for any $m < n$ which shows that even if the algorithm is allowed to store $n-1$ arms in memory at any time, we cannot hope to get a better regret guarantee. This almost matches with the  $\tilde{O}(T^{2/3})$ bound on the expected cumulative regret, obtained by the standard uniform-exploration algorithm. \\ 

\noindent{\bf Best-arm identification:}
We propose an $r$-round ($1\le r \le \log^*n$) adaptive $(\varepsilon,\delta)$-PAC streaming algorithm.
Adaptive algorithms are well-studied in active learning. In $r$-round adaptive streaming algorithm, the arm pulls in each round is decided based on the observed outcomes in the previous rounds, and the best-arm is then output at the end of $r$ rounds. 
Our algorithm stores $O(r)$ arms in memory at any time, and its sample complexity asymptotically matches with the lower bound for any $r$-round adaptive algorithm by \cite{AGARWAL2017}. In particular, when $r = \log^* n$, our algorithm achieves the optimal worst-case sample complexity $O\big(\frac{n}{\varepsilon^2}\log(1/\delta)\big)$ for any best-arm identification algorithm (\cite{EVENDAR2002}) and has space complexity $O(\log^* n)$.

This problem was also studied by \cite{ASSADI2020} and their algorithm was claimed to be an $(\varepsilon,\delta)$-PAC algorithm with optimal sample complexity and $O(\log^* n)$ space complexity. However, we show that due to an oversight in their analysis, the algorithm of \cite{ASSADI2020} is  \emph{not} $(\varepsilon,\delta)$-PAC. In Appendix \ref{subsec:adversarial example}, we construct a family of input instances for which the algorithm will output a non-$\varepsilon$-best arm with probability significantly larger than $\delta$. We note here that for the special case when $r = \log^* n$, our algorithm does provide the guarantees claimed in  \cite{ASSADI2020}.

This leads us to the question of finding  $(\varepsilon,\delta)$-PAC guarantee with optimal sample complexity while using only $O(1)$ arm-memory. Towards this, we propose an algorithm that stores \emph{exactly one} extra arm in the memory. We then show that under the assumption of random-order arrival of arms, our algorithm outputs an $\varepsilon$-best arm with high confidence when the expected rewards of the arms are drawn from some standard distributions. We conclude by experimentally showing that our algorithm performs well on randomly generated input without any assumptions on the arrival order of the arms.

\subsection{Notation}
Let $[k]$ (where $k \in \mathbb{N}$) denote the set $\{1,2,\ldots, k\}$. Let $\log$ denote the binary logarithm. For integers $r\ge0$, and $a\ge1$, $\ilog^{(r)}(\cdot)$ denotes the iterated logarithm of order $r$, i.e., $\ilog^{(r)}(a)=\max\{\log(\ilog^{(r-1)}(a)),1\}$ and $\ilog^{(0)}(a)=a$. Hence, $\ilog^{(\log^* n)}(n)=1$. Let $\mathbb{P}$, $\mathbb{E}$  denote probability and expectation, respectively. 
\section{Model and Problem Definition}
\label{section:model}
An instance of the \MAB problem is defined as the tuple $\langle n, (\mean_i)_{i\in [n]} \rangle$, where $n$ is the number of arms. A pull of $\arm_i$ gives a reward in $[0,1]$ drawn from a distribution with mean $\mean_i \in [0,1]$ that is unknown to the decision-maker beforehand. We study this problem in a \emph{bounded arm-memory} setting where the arms arrive in a stream, and at any time-step, the algorithm can only store a subset of the arms in memory. Any arm that the algorithm wants to pull, either immediately or in the future, has to be present in the memory. An arm that is not present in the memory cannot be pulled.

In the literature, \MAB problems have been  studied with the following objectives: 1) regret minimization, and 2) best-arm identification. Next, we formalize these two notions and their adaptation to our setting.

The regret of a \MAB algorithm can be thought of as the loss suffered by it due to not knowing the reward distributions of the arms beforehand. Let $i^* = \arg\max_{i\in[n]}\mean_i$. Then $\arm_{i^*}$ is the best arm and let $\mean^* = \mean_{i^*}$. The cumulative regret (also called the \emph{pseudo-regret}) of an algorithm over $T$ time-steps is defined as follows:

\begin{definition}[Cumulative Regret]
	Given an instance $\langle n, (\mean_i)_{i\in [n]} \rangle$ of the \MAB problem, the cumulative regret of an algorithm after $T$ rounds is defined as
	$R(T) = \mean^*\cdot T - \sum_{t=1}^T \mean_{i_t}$, 
	where $\arm_{i_t}$ is the arm pulled by the algorithm at time $t \in [T]$, and $\mean^*=\max_{i\in[n]}\mean_i$.
\end{definition}
The expected cumulative regret of an algorithm is defined as $\mathbb{E}[R(T)] = \mean^*\cdot T - \sum_{t=1}^T\mathbb{E}[\mean_{i_t}]$, where the expectation is over the randomness in the algorithm and the distribution of rewards. In the model with bounded arm-memory, the goal is to minimize expected cumulative regret while storing at most $m$ $(< n)$ arms in memory at any time-step. Note that, popular algorithms such as UCB-1 (\cite{AUER2002}) and Thompson sampling (\cite{THOMPSON1933}) store all $n$ arms in memory, i.e., they have space complexity $O(n)$. 

For best-arm identification, the goal of a decision-maker is to output the best arm $\arm_{i^*}$ using the minimum number of arm pulls. In practice, a relaxed goal is to find an arm which is close to the best arm in terms of the expected reward. We formalize this notion below.

\begin{definition}[$\varepsilon$-best arm]
	\label{def:epsilon-optimal arm}
	Given a parameter $\varepsilon \in (0,1)$, $\arm_i$ with mean reward $\mean_i$ is said to be an {\em $\varepsilon$-best arm} if $\mean_i \geq \mean^* - \varepsilon$. Otherwise we call the arm a {\em non-$\varepsilon$-best arm}.
\end{definition}
The reward gap of $\arm_i$ is defined as $\Delta_i = \mean^* - \mean_i$. Without loss of generality, we assume that the best arm is unique, i.e., $\Delta_i > 0$ for all $i \neq i^*$.
\begin{definition}[$(\varepsilon,\delta)$-PAC Algorithm]
	\label{def:PAC algorithm}
	Given an approximation parameter $\varepsilon \in [0,1)$ and a confidence parameter $\delta \in [0,1/2)$, an algorithm $\mathcal{A}$ is said to be an $(\varepsilon,\delta)$-PAC algorithm if it outputs an $\varepsilon$-best arm with probability at least $1 - \delta$.
\end{definition}

Traditionally, the goal in the best-arm identification problem is to design an $(\varepsilon,\delta)$-PAC algorithm that minimizes the total number of arm pulls. Under the  streaming setup, given bounded arm-memory $m$, the goal now is to find an $(\varepsilon,\delta)$-PAC algorithm that minimizes the total number of arm pulls while storing at most $m$ arms in the memory at any time.

Our best-arm identification algorithm in Section \ref{section:round algo} is an $r$-round adaptive streaming algorithm, where in each round $j \in [r]$, only a subset of the arms processed in round $j$ is sent to round $j+1$ and the rest of the arms are discarded from the memory. Additionally, once an arm is discarded, it cannot be pulled in any subsequent rounds, i.e., it is a single-pass algorithm. The set of arms to be sent to round $j+1$ is decided based only on the outcomes in rounds $1$ to $j$. Further, once an arm reaches round $j+1$, the number of times the arm will be sampled in round $j+1$ gets decided before the sampling begins. This number only depends on the round index, i.e., $j+1$ and the outcomes of the pulls of any arm up to the round $j$. All arms in the stream are pulled in round $1$ and the arm output after round $r$ is the best-arm guess of the algorithm. 

The $r$-round adaptive algorithm model was discussed in great detail by \cite{AGARWAL2017}. If $r=1$, then the algorithm is said to be {\em non-adaptive}. The algorithm is said to be {\em fully adaptive} if $r$ is unbounded. If the algorithm is fully adaptive then there is a potential to reduce the sample complexity but the downside of full-adaptivity is that such algorithms are highly sequential. This is because the set of arms to be sampled in a given round can only be determined after we observe the outcomes of pulls of the arms up to the previous round. In contrast, algorithms with only a few rounds of adaptivity enable us to enjoy the benefits of parallelism. 

Some of our results in Section \ref{section: new algo} hold for random-order arrival of arms, which we define next. Let $\langle n, (\mean_i)_{i\in[n]} \rangle$ be an instance of the \MAB problem. Let, $\sigma: [n] \rightarrow [n]$ be a permutation and let $(\arm_{\sigma(i)})_{i\in[n]}$ be the ordering of $(\arm_i)_{i\in[n]}$ under $\sigma$. Define $\mathcal{S}_n = \{\sigma : \sigma \text{~is a permutation of ~}[n]\}$. Under the random-order arrival model, we assume that the arrival order of the arms in the stream is determined by a permutation $\sigma \in \mathcal{S}_n$, which is drawn uniformly at random from the set $\mathcal{S}_n$. The arms arrive in the order in which they appear in the tuple $(\arm_{\sigma(i)})_{i\in[n]}$, i.e., the first arm to arrive in the stream is $\arm_{\sigma(1)}$, followed by $\arm_{\sigma(2)}$, and so on. Random-order arrival is a well-studied model in optimization under uncertainty due to its connectiond with secretary problem and optimal stopping theory \citep{KarlinL15}. The algorithm of \cite{CHAUDHURI2019} also uses an analogous random shuffling of arms.

\section{Regret Minimization}
\label{sec:regretLower}

In this section, we study limitations of bounded arm-memory for regret minimization. 
An adaptation of uniform-exploration algorithm (see \cite{SLIVKINS2019})  achieves expected cumulative regret of $\tilde{O}(T^{2/3})$ with an arm-memory of two. The algorithm  keeps in memory one arm $\arm^*$, called the king, with the best empirical mean $\widehat{\mean}_{\arm^*}$ among the arms seen so far. Whenever a new arm $\arm_i$  arrives,  $\arm_i$ is sampled $(T/n)^{2/3} O(\log T)^{1/3}$ times to obtain its empirical mean $\widehat{\mean}_i$. Then $\widehat{\mean}_i$ is compared with $\widehat{\mean}_{\arm^*}$. If $\widehat{\mean}_{\arm^*} < \widehat{\mean}_i$, then $\arm_i$ becomes the new king, replacing $\arm^*$.  
After the algorithm tries out all the arms, it returns the king as the best-arm and continues to sample it for the rest of the time horizon.

A question left open in (Chaudhuri et al. \citeyear{CHAUDHURI2019}) is to provide a lower bound on the expected cumulative regret of an algorithm with bounded arm-memory. We settle the question by showing that any single-pass algorithm for such a setting incurs at least $\Omega(T^{2/3})$ regret.
Our result is based on $\KL$-divergence, which we define below.
\begin{definition}
Let $\mathtt{\Omega}$ be a finite sample space  and $p, q$  be two probability distributions on $\mathtt{\Omega}$.  $\KL$-divergence is defined as:
\[ \KL(p,q)=\sum_{x \in \mathtt{\Omega}} p(x) \ln(p(x)/q(x)) =\mathbb{E}_p[\ln(p(x)/q(x))].\]
\end{definition}

Now we state some fundamental properties of $\KL$-divergence that will be needed in this section. 

\begin{thm}[\cite{SLIVKINS2019}]
\label{thm:klprop} 
$\KL$-divergence satisfies the following properties:
\begin{itemize}
\item {\em Pinsker's inequality:} For any event $A \subset \omegsamp$, we have $2(p(A)-q(A))^2 \le \KL(p,q)$.
\item {\em Chain rule for product distributions:} Let the sample space be a product $\omegsamp= \omegsamp_1 
\times \omegsamp_2 \times \ldots \times \omegsamp_t$. Let $p, q$ be two distributionas on $\omegsamp$ such that 
$p = p_1 \times p_2 \times \ldots \times p_t$ and $q = q_1 \times q_2 \times \ldots \times q_t$, where $p_j, q_j$ are distributions on $\omegsamp_j$,
for each $j \in [t]$. Then $\KL(p,q)=\sum_{j=1}^t \KL(p_j, q_j)$.
\item {\em Random coins:}
Let $\mathtt{B}_\epsilon$ denote a Bernoulli distribution with mean $(1+\epsilon)/2$.
Then $\KL(\mathtt{B}_\epsilon, \mathtt{B}_0) \le 2 \epsilon^2$ and $\KL(\mathtt{B}_0, \mathtt{B}_\epsilon) \le \epsilon^2$, for all $\epsilon \in (0,1/2)$.
\end{itemize}
\end{thm}

Our main result is the following theorem.

\begin{thm}
In the \MAB setting, fix the number of arms $n$ and the time horizon $T$. For any online \MAB algorithm, if we are allowed to store at most $m < n$ arms, then there exists a problem instance such that $\mathbb{E}[R(T)] \geq \Omega({T^{2/3}}/{m^{7/3}})$ 
\end{thm}
\begin{proof}
	We consider $0$-$1$ rewards and the following family of problem instances $\{\mathcal{I}_j:j\in\{0,\ldots,m\}\}$ each containing $n$ arms, with parameter $\epsilon>0$ (where $\epsilon=\frac{1}{m^{1/3}T^{1/3}}$):\
	
	\begin{align*}
	\mathcal{I}_0= \left\{ \begin{array}{rcl}
	\mean_i &= 1/2,~~~~~~~  &\mbox{for~~} i\neq n;\\
	\mean_i &= 1,~~~~~~~~~~ & \mbox{for~~}  i=n.
	\end{array}\right.\\
	\forall j\in[m],~\mathcal{I}_j = \left\{ \begin{array}{rcl}
	\mean_i & =(1+\epsilon)/2,& \mbox{for~~}  i=j; \\ 
	\mean_i & = 1/2,~~~~~~~ &  \mbox{for~~}  i \neq j.
	\end{array}\right.
	\end{align*}
	\noindent In the above instances, $\mean_i$ denotes the expected reward of $\arm_i$, the $i$-th arm  to arrive in the stream.\\
	
Note that a deterministic algorithm that directly stores the first $m$ arms in the memory and has the least expected regret among all such deterministic algorithms, can not have a worse regret compared to any other algorithm that processes the first $m$ arms in some different manner. This is because the processing of any such algorithm can be replicated by an algorithm which directly stores the first $m$ arms.  So  we fix a deterministic algorithm $\mathcal{A}$ which directly stores the the first $m$ arms in the memory.
	 
We next set up the sample space. Let $L = {1}/({4m^2\epsilon^2})$. Further, let $(r_s(i): i\in[m], s \in [L])$ be a tuple of mutually independent Bernoulli random variables where $r_s(i)$ has expectation $\mean_i$. We interpret $r_s(i)$ as the reward obtained when $\arm_i$ is pulled for the $s$-th time and the tuple is called the \emph{rewards table}. The sample space is then expressed as $\omegsamp = \{0,1\}^{m\times L}$ and any $\omega \in \omegsamp$ can be interpreted as a realization of the rewards table.
	
Each instance $\mathcal{I}_j$, where $j \in \{0,\ldots, m\}$, defines a distribution $P_j$ on $\omegsamp$ as follows:
\begin{equation*}
P_j(A)= \mathbb{P}[A~|~\mathcal{I}_j], \text{~~~~~for each $A\subseteq \omegsamp$}
\end{equation*}
	
	Given an instance $\mathcal{I}_j$ where $j\in \{0,\ldots, m\}$, let $P_j^{i,s}$ be the distribution of $r_s(i)$ under this instance. Then we have that $P_j = \prod_{i\in[m],s\in [L]} P_{j}^{i,s}$.
	
	Let $S_{t} \subseteq \{\arm_1, \arm_2, \dots, \arm_m\}$ denote the subset of first $m$ arms which are discarded from memory till (and including) time step $t$ by the algorithm $\mathcal{A}$. If $\arm_i$ is discarded before the algorithm begins pulling arms, call this time step $0$, then we include $\arm_i$ in the set $S_0$ where $i\in[m]$. As time horizon $T$ is fixed and we will eventually discard all arms at the end of time horizon, we can assume that $S_T= \{\arm_1, \arm_2, \dots, \arm_m\}$. For all $\omega \in \omegsamp$, let $T'_\omega = \arg\min_{0\leq t \leq T}\{t : S_t \neq \emptyset\}$, i.e., $T'_{\omega}$ is the number of time steps since the beginning of the algorithm $\mathcal{A}$ when some arm in $\{\arm_1, \arm_2, \dots, \arm_m\}$ is discarded from memory for the first time. Let $A_1=\{\omega \in \omegsamp : ~T'_\omega\leq L\}$ be the set of reward realizations for which $T'_{\omega} \leq L$. Now fix some arm $i \in \{\arm_1, \arm_2, \dots, \arm_m\}$. Define $A_2^{i}=\{\omega \in \omegsamp: \arm_i \in S_{T'_\omega}\}$ to be the event that the $\arm_i$ belongs to $S_{T'_{\omega}}$. Now, let $A^{i} = A_1 \cap A_2^{i}$ be the set of reward realizations such that $\forall \omega \in A^i$, $T'_\omega \leq L$ and $\arm_i$ is discarded from memory at the time step $T'_\omega$. Also, for any event $A\subseteq \omegsamp$, let $\overline{A}=\omegsamp \setminus A$. 	

Now we have the following observation for instance  $\mathcal{I}_0$. 
\begin{obs}
\label{obs1}
If $\omega \in \overline{A_1}$, then the algorithm $\mathcal{A}$ would incur regret of $\Omega(\frac{1}{m^2  \epsilon^2})$ on the instance $\mathcal{I}_0$. 
\end{obs}

Let $i'=\arg \max_{i\in[m]}P_0(A^i)$.
We obtain the next observation due to the fact that  $L=o(T)$.
\begin{obs}
\label{obs2}
For all $\omega \in A^{i'}$, the regret for instance $\mathcal{I}_{i'}$ is at least $\frac{\epsilon (T-T'_\omega)}{2}=\Omega(\epsilon T)$. 
\end{obs}
 
Now  we will prove the following inequality which will be useful in our analysis:

\begin{equation}
\label{eqKL}
m\cdot P_{i'}(A^{i'})+P_0(\overline{A_1})\geq \frac{1}{4}.
\end{equation}

The above inequality is trivially true if $P_0(\overline{A_1})\ge 1/4$. Therefore, let us assume $P_0(\overline{A_1})\le 1/4$, i.e., $P_0(A_1)\geq3/4$. 
Then  $P_0(A^{i'}) \geq {3}/{(4m)}$, by averaging argument.
Using Theorem \ref{thm:klprop} for distributions $P_{0}$ and $P_{i'}$, we obtain:
	\begin{align*}
	2(P_{0}(A^{i'})-P_{i'}(A^{i'}))^2	& \leq \KL(P_{0},P_{i'}) \tag{by Pinsker's inequality}\\
	& = \sum_{i\in[m]}\sum_{t=1}^{L} \KL(P_0^{i,t},P_{i'}^{i,t}) \tag{by chain rule}\\
	& = \sum_{i\in[m] \setminus \{i'\}} \sum_{t=1}^{L} \KL(P_0^{i,t},P_{i'}^{i,t}) +\sum_{t=1}^{L} \KL(P_0^{i',t},P_{i'}^{i',t})\\
	& \leq 0 + L\cdot 2\epsilon^2.
	 \end{align*} 
	 
\noindent In the last inequality, the first term of the summation is zero because all arms $\arm_i$, where $i \in[m]\setminus \{i'\}$, have identical reward distributions under instances $\mathcal{I}_0$ and $\mathcal{I}_{i'}$. To bound the second term in the summation, we use the last property from Theorem \ref{thm:klprop}. Thus we have, 
$P_0(A^{i'})- P_{i'}(A^{i'}) \leq \epsilon \sqrt{L}$.
Hence, $P_{i'}(A^{i'}) \geq P_0(A^{i'})-\epsilon \sqrt{L} \geq ({3}/{(4m)})-({1}/{(2m)}) = {1}/{(4m)}$.
Here, we use  $P_0(A^{i'})\geq {3}/{(4m)}$ and  $L= 1/(4 m^2 \epsilon ^2)$.
Hence, if $P_0(\overline{A_1})\leq\frac{1}{4}$ then $m\cdot P_{i'}(A^{i'})\geq \frac{1}{4}$. 
This proves Inequality \eqref{eqKL}. 
	
Now suppose that we choose an instance $\mathcal{I}$ uniformly at random from the family of $m+1$ instances $\{\mathcal{I}_j:j\in\{0,\ldots,m\}\}$, i.e., we have $\mathbb{P}\big[\mathcal{I} = \mathcal{I}_{j}\big] = {1}/{(m+1)}$ for any $j\in\{0,\ldots,m\}$.  Then, 
	\begin{align*}
	\mathbb{E}[R(T)] &\geq  ~\mathbb{P}\big[\mathcal{I} = \mathcal{I}_{i'}\big]\cdot P_{i'}(A^{i'})\cdot \Omega(\epsilon T) + \mathbb{P}\big[\mathcal{I} = \mathcal{I}_{0}\big]\cdot P_0(\overline{A_1}) \cdot \Omega\Big(\frac{1}{m^2\cdot \epsilon ^2}\Big)	\\
&\geq  ~\frac{1}{m+1}\cdot (m\cdot P_{i'}(A^{i'}))\cdot \Omega\Big(\frac{\epsilon T}{m}\Big) + \frac{1}{m+1}\cdot P_0(\overline{A_1}) \cdot \Omega\Big(\frac{1}{m^2\cdot \epsilon ^2}\Big)	\\
	& = ~\Big(m\cdot P_{i'}(A^{i'}) +P_0(\overline{A_1})\Big)\cdot \Omega\Big(\frac{T^{2/3}}{m^{7/3}}\Big)\\
	&=\Omega\Big(\frac{T^{2/3}}{m^{7/3}}\Big)
	\end{align*}

The first inequality follows from Observation \ref{obs1} and \ref{obs2}. In the last equality, we have used Inequality \eqref{eqKL}.
Note that the expectation is taken over that choice of input instance and randomness in reward.
\end{proof}

Since a randomized algorithm is a distribution over deterministic algorithms, the above result also holds for any randomized algorithm. Also, by slight modification to the above family of instances, we can show the same lower bound on the expected cumulative regret even under the assumption that arms arrive in a random-order (see Appendix \ref{appendix: random regret} for details).

\section{Best-arm Identification}
\label{section:round algo}
In this section, we design an $(\varepsilon,\delta)$-PAC algorithm which minimizes the total number of arm pulls while storing at most $m < n$ arms in memory.

Towards this goal, we propose a general algorithmic framework (Algorithm \ref{r_arms}), which is an $r$-round adaptive streaming algorithm for $r \in [\log^*n]$. This algorithm has the optimal sample complexity for any $r$-round adaptive streaming algorithm (refer to Appendix \ref{r-lowerrr-bnd} for a detailed discussion) and stores $O(r)$ arms in memory at any time step. 
Intuitively, 
in each of the $r$ rounds, we keep a running best arm candidate (denoted $\arm^*_i$ for round $i$) in the memory. Once we see {\em sufficient} number ($= c_i$ for round $i$) of arms in a round, we send $\arm_i$ to the next round. 
At each round $i$, only one out of every $c_i$ arms is sent to round $i+1$. For higher round indices, the number of arms reaching that round decreases rapidly. Hence, each arm can be sampled more number of times for a more refined comparison without affecting the sample complexity.

Our algorithm is related to the recent work by (Assadi et al. \citeyear{ASSADI2020}).  They proposed an $(\varepsilon,\delta)$-PAC algorithm for this setting which has optimal sample complexity and stores at most $\log^* n$ arms in memory at any time step. We remark here that their analysis has an oversight, due to which their algorithm will output a non-$\varepsilon$-best arm with probability much greater than $\delta$ for some input sequences. 
We refer the reader to Appendix \ref{subsec:adversarial example} for a detailed discussion. 
Unlike their algorithm, whenever a new arm arrives we do not again sample the stored best arm for comparison. 
Instead, we reuse the stored empirical mean of the best arm for the comparison.
Due to this subtle difference, for $r = \log^* n$, our algorithm does in fact provide the guarantees claimed in (Assadi et al. \citeyear{ASSADI2020}). In this section we use the terms round and level interchangeably. For simplicity during the analysis, we ignore the {\em ceil} in the expression of $s_\ell$ and $c_\ell$.
 
\begin{algorithm}[ht!]
\caption{}
\label{r_arms}
\begin{algorithmic}[1]
\STATE $\{ \varepsilon_\ell \}_{\ell=1}^{r}: \varepsilon_\ell={\varepsilon}/{2^{\ell+1}}.$ \\\texttt{//Intermediate gap parameter.}
\STATE $\{ \beta_\ell \}_{\ell=1}^{r}:  \beta_\ell={1}/{\varepsilon^2_\ell}.$
\STATE $\{ s_\ell \}_{\ell=1}^{r}:  s_\ell = \lceil2\beta_\ell\big(\ilog^{(r+1-\ell)}(n)+\log(\frac{2^{\ell+2}}{\delta})\big)\rceil.$
\texttt{//Samples per arm in level $\ell$.}
\STATE $\{ c_\ell \}_{\ell=1}^{r}:  c_\ell=\lceil\ilog^{(r-\ell)}(n)\rceil.$\\
\texttt{//Number of  arms in level $\ell$.}
\STATE Counters: $C_1,C_2,\ldots,C_r$ initialized to 0.
\STATE Stored arms: $\arm^*_1, \arm^*_2, \ldots, \arm^*_r$, where $\arm^*_\ell$ is the arm with the highest empirical mean at  $\ell$-th level.
\STATE Stored empirical means: $\mean^*_1,\mean^*_2,\ldots,\mean^*_r$, where $\mean^*_\ell$ is the highest empirical mean of $\ell$-th level, initialized to 0.
\WHILE {a new arm $\arm_i$ arrives in the stream}\label{st:whilee}
\STATE Read $\arm_i$ to memory.
\STATE \textbf{\underline{Modified Selective Promotion:}} Starting from level $\ell=1$:
\STATE Sample $\arm_i$ for $s_\ell$ times and compare its empirical mean with $\mean^*_\ell$.\label{next level11}
\STATE If $\widehat{\mean}_{\arm_i} < \mean^*_\ell$, drop $\arm_i$. Otherwise, replace $\arm^*_\ell$ with $\arm_i$ and make $\mean^*_\ell$ equal to $\widehat{\mean}_{\arm_i}$.
\STATE Increase $C_\ell$ by 1.
\STATE If $C_\ell=c_\ell$ and $r=1$, \textbf{return} $\arm^*_\ell$ as the selected arm and terminate the Algorithm.
\STATE If $C_\ell=c_\ell$, make $C_\ell$ and $\mean^*_\ell$ equal to 0, send $\arm^*_\ell$ to the next level by calling Line \ref{next level11} with $(\ell=\ell+1)$.
\ENDWHILE
\STATE If there is any arm which is stored in a level below $r$, then promote it to level $r$ and sample it for $s_r$ times. Let $\arm^*_{\ell'}$ be the arm with highest empirical mean $\mean^*_{\ell'}$ among the arms which were sampled, where $\ell'\in [r-1]$.\label{special:st}
\STATE \textbf{If} $\mean^*_r > \mean^*_{\ell'}$ \textbf{then } Return $\arm^*_r$
\STATE \textbf{Else } Return $\arm^*_{\ell'}$
\end{algorithmic}
\end{algorithm}

\begin{thm}
	\label{theorem:fix analysis}
	Algorithm \ref{r_arms} is an $(\varepsilon,\delta)$-PAC algorithm with sample complexity $O(\frac{n}{\varepsilon^2}\cdot (\ilog^{(r)}(n)+\log(\frac{1}{\delta})))$ and space complexity $O(r)$ where $1\leq r \leq \log^*(n)$.
\end{thm}
\begin{proof}
	We prove this theorem using the following lemmas. Lemma \ref{lemma:fix sample complexity} shows that the sample complexity of our algorithm is $O(\frac{n}{\varepsilon^2}\cdot (\ilog^{(r)}(n)+\ln(\frac{1}{\delta})))$. Lemma \ref{lemma: fix correctness} gives the proof of correctness. Further, at each of the $r$ levels we store a single arm along with its empirical mean, implying the space complexity of our algorithm.
\end{proof}
\begin{lem}
	\label{lemma:fix sample complexity}
	The sample complexity of the algorithm is $O(\frac{n}{\varepsilon^2}\cdot (\ilog^{(r)}(n)+\ln(\frac{1}{\delta})))$.
\end{lem}
\begin{proof}
If $r=1$, then the total number of samples is $n\cdot s_1= O(\frac{n}{\varepsilon^2}\cdot (\ilog^{(r)}(n)+\log(\frac{1}{\delta})))$. Let $c_0 :=1$. So for the rest of the analysis we assume that $r\geq2$ and define $c_i:=2^{c_{i-1}},\forall i>r$. Note that since $2\leq r\leq \log^*(n)$, $c_2\geq 2$ and $c_i= 2^{c_{i-1}}$, $\forall i\geq 3$. 
For any level $\ell-1$, we send one arm from level $\ell-1$ to level $\ell$ for every $c_{\ell-1}$ arms seen (this is excluding the arms sampled in Step \ref{special:st}). Hence during the {\em Modified Selective Promotion}, the number of arms that can reach any level $\ell$ is at most ${n}/({\prod_{i=0}^{\ell-1}c_i})$. Each arm arriving at level $\ell$ is pulled exactly $s_\ell$ times. Also note that we can sample up to $(r-1)\cdot s_r$ times in the Step \ref{special:st}. Since, we have $r$ levels, the total number of samples can be bounded as:
\begin{align*}
	&\sum_{\ell=1}^{r} \frac{n}{\prod_{i=0}^{\ell-1}c_i} \cdot s_\ell+(r-1)\cdot s_r \\
	 \leq &  \sum_{\ell=1}^{r} \frac{2n\beta_\ell\big({\ilog}^{(r+1-\ell)}(n)+\log(\frac{2^{\ell+2}}{\delta})\big)}{\prod_{i=0}^{\ell-1}c_i}+r\cdot s_r\\
	\leq & n\cdot s_1+ r\cdot s_r+ \frac{2n}{\varepsilon^2} \cdot \sum_{\ell=2}^{r} 4^{\ell+1} \cdot \bigg( \frac{{\ilog}^{(r+1-\ell)}(n)}{c_{\ell-1}\cdot c_{\ell-2}} + \frac{2\ell\cdot \log(\frac{2}{\delta})}{c_{\ell-1}\cdot c_{\ell-2}}\bigg)\\
	&\text{\big(Since, $\prod_{i=0}^{\ell-1}c_i \geq c_{\ell-1}\cdot c_{\ell-2}$, $\beta_{\ell}={4^{\ell+1}}/{\varepsilon^2},$ and }\text{$\log({2^{\ell+2}}/{\delta})\leq \log({2^{2\ell}}/{\delta^{2\ell}})\leq 2\ell \log({2}/{\delta})$\big)}\\
	\leq & n\cdot s_1+ r\cdot s_r+({2\cdot 4^5\cdot n}/{\varepsilon^2})\sum_{\ell=2}^{\infty}({4^{\ell - 4}}/{c_{\ell-2}})+ ({2^2\cdot 4^5\cdot n }/{\varepsilon^2})\cdot \log({2}/{\delta})\sum_{\ell=2}^{\infty}({4^{\ell - 3}}/{c_{\ell-1}})\\
	&\text{\big(Since, $c_{\ell-1}={\ilog}^{(r+1-\ell)}(n), \text{ and } ({\ell}/{c_{\ell-2}})\leq 4$\big)}\\
	\leq & n\cdot s_1+ r\cdot s_r +{O(n/\varepsilon^2)} \big(1 +  \log({2}/{\delta})\big)\\ 
	&\text{\big(Since, $\sum_{\ell=0}^{5}\frac{4^{\ell-2}}{c_\ell} = O(1)$, $\sum_{\ell=6}^{\infty}\frac{4^{\ell-2}}{c_\ell} < \sum_{\ell=6}^{\infty}\frac{4^{\ell-2}}{8^{\ell-2}} <1$\big)}\\
	\leq & O({n}/{\varepsilon^2}) \cdot ({\ilog}^{(r)}(n)+\log({1}/{\delta}))\\
	& \text{\big(Since, $r\cdot s_r= O(n/\varepsilon^2)(1+\log(1/\delta))$ and }\text{$n\cdot s_1= O({n}/{\varepsilon^2})\cdot ({\ilog}^{(r)}(n)+\log({1}/{\delta}))\big)$.}
\end{align*}

Hence, we have that the sample complexity is $O(\frac{n}{\varepsilon^2}\cdot (\ilog^{(r)}(n)+\log(\frac{1}{\delta})))$.
\end{proof}

We now use the following two claims to prove the correctness of our algorithm. The proofs of both these claims follow from the application of Hoeffding's inequality followed by taking a union bound over the number of arms that are compared at a given level (see Appendix \ref{appendix: omitted proofs} for the detailed proof).
\begin{restatable}{claim}{RoundAlgoLevelSuboptimality}
	\label{claim:level_suboptimality}
	For any level $\ell<r$, let $\arm'_\ell$ be the best arm to ever reach this level. Then, with probability at least $1 - \frac{\delta}{2^{\ell+1}}$, an arm with reward gap at most $\varepsilon_\ell$ from $\arm'_\ell$ is sent to level $\ell + 1$ or is sampled in  Step \ref{special:st}.
\end{restatable}

\begin{restatable}{claim}{RoundAlgoLevelSuboptimalityTwo}	
	\label{claim:level_suboptimality2}
	Let $\arm'_r$ be the best arm among the arms which reached level $r$ including the arms which were sampled in  Step \ref{special:st} of Algorithm \ref{r_arms}. Then, with probability at least $1 - \frac{\delta}{2^{r+1}}$, an arm with reward gap at most $\varepsilon_r$ from $\arm'_r$ is returned by the Algorithm.
\end{restatable}


\begin{lem}
	\label{lemma: fix correctness}
	With probability at least $1-\delta$, the arm selected by the algorithm is an $\varepsilon$-best arm.
\end{lem}
\begin{proof}
Let the best arm be $\arm^*$. By union bound and Claim \ref{claim:level_suboptimality}, the probability that an arm  with reward gap at most $\sum_{i=1}^{r-1}\varepsilon_i$ from $\arm^*$ does not reach either level $r$ via Modified Selective Promotion or is not sampled in Step \ref{special:st} is upper bounded by $\sum_{\ell = 1}^{r-1} \frac{\delta}{2^{\ell+1}}$.  Given this upper bound and Claim \ref{claim:level_suboptimality2}, the probability that an arm  with reward gap at most $\sum_{i=1}^{r}\varepsilon_i$ from $\arm^*$ is not returned by the Algorithm is upper bounded by:
$ \sum_{\ell = 1}^{r} \frac{\delta}{2^{\ell+1}} \leq \delta \sum_{\ell=1}^\infty \frac{1}{2^{\ell+1}} = \frac{\delta}{2} < \delta$.


Now with probability at least $1-\delta$, an arm with reward gap at most:
$\sum_{\ell = 1}^{r} \varepsilon_\ell = \sum_{\ell = 1}^{r} \frac{\varepsilon}{2^{\ell + 1}} \leq \frac{\varepsilon}{2}< \varepsilon$
\noindent from $\arm^*$ is returned by the algorithm. This concludes the proof of this lemma.
\end{proof}

\section{Towards Constant Arm-Memory}
\label{section: new algo}
In this section, we take a step towards designing an $(\varepsilon,\delta)$-PAC algorithm which has optimal sample complexity, while using only $O(1)$ arm-memory. 

\cite{ASSADI2020} have  proposed an algorithm which stores only one extra arm, assuming $\Delta$ is known. 
They maintain a candidate best-arm ({\em king}) and assign it a certain budget, essentially denoting the number of permissible arm pulls. For each arriving arm, both king and the new arm are sampled for some number of times, if needed in multiple levels, until either king wins against the new arm (by having a higher empirical mean at one of the levels) or the budget of the king is exhausted. If the budget is exhasuted then the  king is replaced with the new arm and  the process is repeated. 
The number of samples and budget is proportional to $1/\Delta^2$ and this careful choice of budget ensured 
 smaller sample complexity. 
However, when $\Delta$ is not known a similar approach will not work. 
See Appendix \ref{assadi_type} for more details. 

Inspired by their framework, we propose Algorithm \ref{algo:new_algo} for the case when $\Delta$ is not known to the algorithm.
In Step \ref{newchal},  we go to the next level of challenge subroutine only if there is a good chance that the newly arrived arm is significantly better compared to the king. Intuitively, we ensure two properties: 
(i) king only lose to an arm that has significantly better mean compared to the king, and (ii) when the true best-arm arrives, it can only lose to a king if their means are quite close. 
See Appendix \ref{randomorder} for more details. 


\begin{algorithm}[ht!]
	\caption{}
	\label{algo:new_algo}
	\begin{algorithmic}[1]
		\STATE $\{s_\ell\}^{\infty}_{\ell=1}:s_\ell=\lceil\frac{2}{(\frac{\varepsilon}{200})^2}\cdot\ln\left(\frac{4}{\delta}\right)\cdot 3^\ell\rceil$.
		\STATE $b:=\lceil\frac{2}{(\frac{\varepsilon}{200})^2}\cdot C \cdot \ln\left(\frac{4}{\delta}\right)+s_1\rceil$. 
		\texttt{//Here, $C$ is a large constant that we choose later.}
		\STATE Let \textbf{king} be the first available arm and set its budget $\phi:=\phi(\textbf{king})=0$.
		\WHILE {A new arm $\arm_i$ arrives in the stream}\label{assadi-type:st1}
		\STATE Increase the budget $\phi(\textbf{king})$ by $b$. 
		\STATE \textbf{\underline{Challenge subroutine:}} For level $\ell=1$ to $+\infty$ :
		\STATE If $\phi(\textbf{king})<s_\ell$: we declare $\textbf{king}$ defeated, make $\arm_i$ the king, initialize its budget to 0 and go to Step \ref{assadi-type:st1}.
		\STATE Otherwise, we decrease $\phi(\textbf{king})$ by $s_\ell$ and sample both \textbf{king} and $\arm_i$ for $s_\ell$ times.
		\STATE Let $\widehat{\mean}_{king}$ and $\widehat{\mean}_i$ denote the empirical means of \textbf{king} and $\arm_i$ in this trial.
		\STATE \label{newchal} If $\widehat{\mean}_{king} > \widehat{\mean}_i-0.495\varepsilon$, we declare \textbf{king} winner and go to the next arm in the stream; otherwise we go to the next level of challenge (increment $\ell$ by one).
		\ENDWHILE
		\STATE Return $\textbf{king}$ as the selected best-arm.   
	\end{algorithmic}
\end{algorithm}

We experimentally show that  this algorithm performs well on randomly generated input and provide a theoretical justification for these experimental results as follows.

\subsection{Random Order Arrival}
\begin{definition}
Let the P.D.F.~and C.D.F.~of a distribution $\mathcal{D}$ be $g(x)$ and $F(x)$, respectively. Then the truncated distribution $\mathcal{D}$ with support $(a,b]$ is a distribution where P.D.F. and C.D.F.~are $\big\{\frac{g(x)}{F(b)-F(a)}\text{ for }x\in(a,b], 0 \text{ for } x\notin(a,b]\big\}$ and $\big\{\frac{F(x)-F(a)}{F(b)-F(a)}\text{ for } x\in(a,b],1\text{ for } x>b,0 \text{ for }x\leq a\big\}$, respectively. 
\end{definition}

In the following theorem we show that if the arms arrive in random order and the means of arms come from various common distributions, then Algorithm \ref{algo:new_algo} is successful with reasonable probability. 
See Appendix \ref{randomorder} for the proof. 
 
\begin{thm}\label{thm:rand}
Let the means of the arms come from one of the following distributions with support in (0,1] : uniform, truncated normal, truncated lognormal, truncated exponential, beta, truncated gamma, truncated weibull. Then, under random order arrival, asymptotically (i.e., when $n \rightarrow \infty$) the probability that Algorithm \ref{algo:new_algo} returns an $\varepsilon$-best arm is greater than or equal to $0.9(1-\delta)$, $\forall \varepsilon\leq {1}/{10}$ .
\end{thm}

\subsection{Randomly generated input stream}
We now consider finding the $\varepsilon$-best arm when the means of $n$ arms are  i.i.d. samples from a distribution with C.D.F. $F(x)$. Let $F_n(x)$ be the empirical distributions of the means of these  $n$ arms. Then, we have that $\sup_{x\in\mathbb{R}}|F_n(x)-F(x)| \rightarrow 0$ as $n \rightarrow \infty$, due to Glivenko-Cantelli theorem. For practical purposes, $F(x) \approx F_n(x)$ for $n\geq 10^5$. 
Due to Theorem \ref{thm:rand}, for certain well-known distributions, under random order arrival, Algorithm \ref{algo:new_algo} returns an $\varepsilon$-best arm with probability at least $0.9(1-\delta)$, $\forall \varepsilon\leq\frac{1}{10}$. This implies that for any set of such $n$ arms, for at least $\frac{9}{10}\cdot n!$ out of a total $n!$ permutations, Algorithm \ref{algo:new_algo} returns an $\varepsilon$-best arm with probability at least $(1-\delta)$. 
Hence, the probability that for a randomly generated input stream, Algorithm \ref{algo:new_algo} will output an $\varepsilon$-best arm is at least $0.9(1-\delta)$. Here, the probability is calculated over all possible input streams.
\subsection{Experimental evaluation}
We now give experimental evidence that in practice Algorithm \ref{algo:new_algo} returns an $\varepsilon$-arm with high confidence, even when we reduce the number of samplings by a factor of 40000.
In the experiments, we run the algorithm on $R$ different instances and for each instance \arn{means of} $n$ arms were sampled from a distribution $\mathcal{D}$ with support in $(0,1]$,  mean $=\mu$ and variance $=\sigma^2$. Also we set $C=117, \varepsilon={1}/{10}, \delta={1}/{10}$. 
\ari{After we obtain the mean of an arm, the arm has Bernoulli reward distributions with that mean.}
In Figure (\ref{fig1a}), $R=1000$, $n=10^6$, $\mathcal{D}$ is the truncated normal distribution, $\mu=1/2$, $\sigma^2=1$. In Figure (\ref{fig1b}), $R=100$, $n=10^5$, $\mathcal{D}$ is the truncated normal distribution, $\mu=1/2$, $\sigma^2=1/2$. In Figure (\ref{fig1c}), $R=100$, $n=10^6$, $\mathcal{D}$ is the uniform distribution, $\mu=1/2$, $\sigma^2={1}/{12}$. In Figure (\ref{fig1d}), $R=1000$, $n=10^5$, $\mathcal{D}$ is the uniform distribution, $\mu=1/2$, $\sigma^2={1}/{12}$. 
See Appendix \ref{exp:dist123} for the experimental evaluation for more distributions. 
In all these cases, we almost always return an arm with mean within at most 0.05 ($<1/10=\varepsilon$) from the mean of the best-arm. 

\begin{figure}
	\centering
	\begin{subfigure}[b]{0.24\columnwidth}
		\centering
		\includegraphics[scale=0.26]{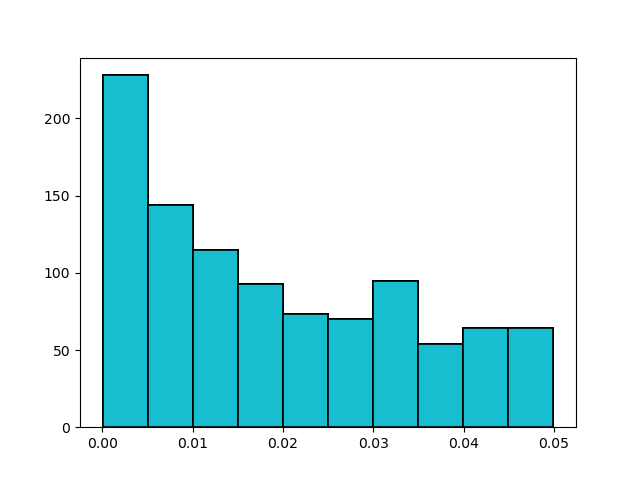}
		\caption[RandomInstance]%
		{{\small Instance 1}}    
		\label{fig1a}
	\end{subfigure}
	\hfill
	\begin{subfigure}[b]{0.24\columnwidth}  
		\centering 
		\includegraphics[scale=0.26]{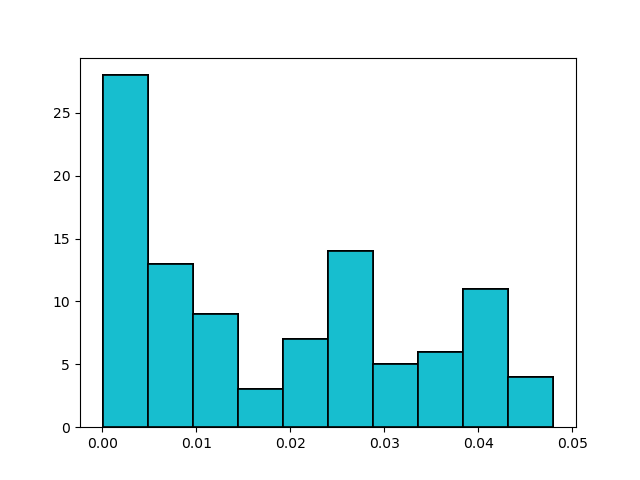}
		\caption[]%
		{{\small Instance 2}}    
		\label{fig1b}
	\end{subfigure}
	\hfill
	\begin{subfigure}[b]{0.24\columnwidth}   
		\centering 
		\includegraphics[scale=0.26]{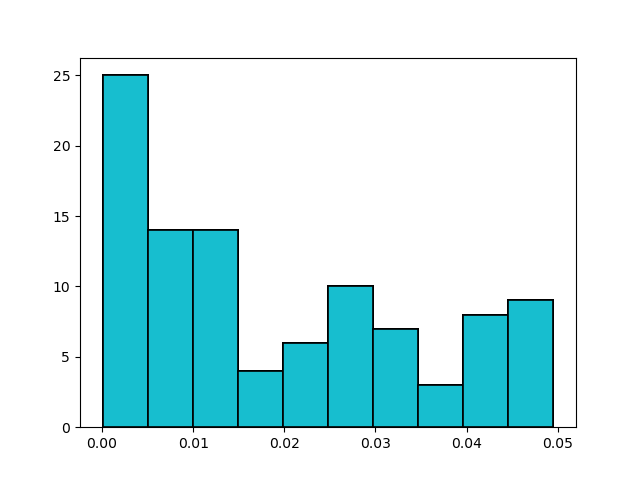}
		\caption[]%
		{{\small Instance 3}}    
		\label{fig1c}
	\end{subfigure}
	\hfill
	\begin{subfigure}[b]{0.24\columnwidth}   
		\centering 
		\includegraphics[scale=0.26]{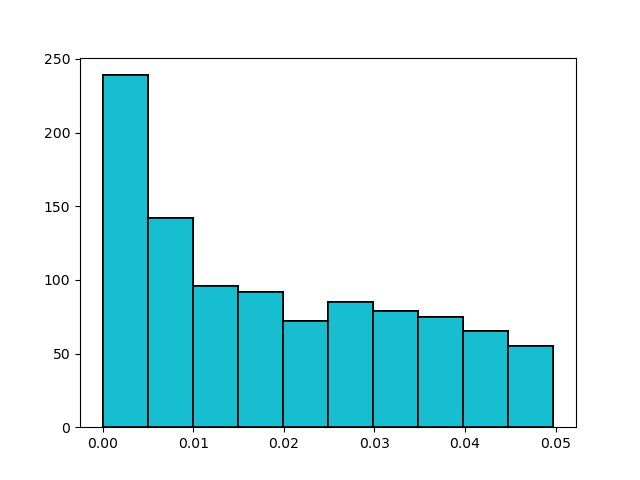}
		\caption[]%
		{{\small Instance 4}}    
		\label{fig1d}
	\end{subfigure}
	\caption[]
	{\small \textbf{X-axis}: Gap between means of the best arm and the arm returned by Algorithm \ref{algo:new_algo}, \textbf{Y-axis}: Count of such arms.} 
	\label{fig:mean and std of nets}
\end{figure}


 
\section{Conclusion}
\label{section: conclusion}
We study the \MAB problem  with bounded arm-memory where the arms arrive in a stream. We study two standard objectives: regret minimization and best-arm identification. 

For regret minimization, \cite{PEKOZ2003} show that when $m=1$, a single-pass algorithm can have linear regret unless the reward distributions of arms satisfy some additional conditions. 
\cite{LIAU2018} conjecture an instance-dependent lower bound on the expected regret. (Chaudhuri et al.~\citeyear{CHAUDHURI2019}) leave it as an open problem to prove a lower bound on the expected regret of any bounded arm-memory algorithm. Our first result shows a lower bound of $\Omega\big({T^{2/3}}/{m^{7/3}}\big)$ on the expected cumulative regret for single-pass \MAB algorithms with bounded arm-memory. The lower bound holds for any $m < n$. 
This shows a nice dichotomy for $T >> n$, as one can obtain  expected cumulative regret of $\tilde{O}\big({\sqrt{nT}})$ by standard UCB1 algorithm, where we are allowed to store $n$ arms.
Note that the question of proving a lower bound on the regret of multi-pass algorithms remains open.



The best-arm identification problem in a streaming model has  been studied by (Assadi et al. \citeyear{ASSADI2020}). We propose an $r$-round adaptive $(\varepsilon,\delta)$-PAC streaming algorithm that uses $O(r)$ arm-memory ($r \in [\log^* n]$) and has tight sample complexity. We also propose a best-arm identification algorithm that stores exactly one extra arm in the memory and outputs an $\varepsilon$-best arm with high confidence for most standard distributions. We note  that Algorithm \ref{algo:new_algo} and the techniques used in its analysis for random order arrival may help in resolving the  problem with adversarial order arrival of arms.

Another interesting question is to find the \emph{top-k} arms in the stream. For this problem, under the assumption that $\Delta$ is known, (Assadi et al. \citeyear{ASSADI2020}) propose an algorithm that has optimal sample complexity and stores exactly $k$ arms in the memory at any time-step. Note that our algorithm in Section \ref{section:round algo} can be directly extended to find the \emph{top-k} arms having optimal $r$-round sample complexity and $O(kr)$ space complexity, even when $\Delta$ is not known beforehand. It would be interesting to see if these ideas can be extended to an $(\varepsilon,\delta)$-PAC algorithm that finds the \emph{top-k} arms using the optimal number of samples, while improving the space complexity to say, $O(k + r)$.
\bibliographystyle{apalike}
\bibliography{references}

\begin{thebibliography}{}

\bibitem[Garivier and Capp{\'e}, 2011]{Garivier2011}
  Garivier, Aur{\'e}lien and Capp{\'e}, Olivier.
\newblock The KL-UCB algorithm for bounded stochastic bandits and beyond.
  
\newblock In {\em Proceedings of the 24th annual conference on learning theory}, pages 359--376.


\bibitem[Agarwal et~al., 2017]{AGARWAL2017}
Agarwal, A., Agarwal, S., Assadi, S., and Khanna, S. (2017).
\newblock Learning with limited rounds of adaptivity: Coin tossing, multi-armed
  bandits, and ranking from pairwise comparisons.
\newblock In {\em Conference on Learning Theory}, pages 39--75.

\bibitem[Agrawal and Goyal, 2012]{AGRAWAL12}
Agrawal, S. and Goyal, N. (2012).
\newblock Analysis of thompson sampling for the multi-armed bandit problem.
\newblock In {\em Conference on Learning Theory}, pages 39--1.

\bibitem[Alon et~al., 1996]{ALON1996}
Alon, N., Matias, Y., and Szegedy, M. (1996).
\newblock The space complexity of approximating the frequency moments.
\newblock In {\em Proceedings of the twenty-eighth annual ACM symposium on
  Theory of Computing}, pages 20--29.

\bibitem[Assadi and Wang, 2020]{ASSADI2020}
Assadi, S. and Wang, C. (2020).
\newblock Exploration with limited memory: streaming algorithms for coin
  tossing, noisy comparisons, and multi-armed bandits.
\newblock In {\em Proceedings of the 52nd Annual ACM SIGACT Symposium on Theory
  of Computing}, pages 1237--1250.

\bibitem[Auer et~al., 2002]{AUER2002}
Auer, P., Cesa-Bianchi, N., and Fischer, P. (2002).
\newblock Finite-time analysis of the multiarmed bandit problem.
\newblock {\em Machine learning}, 47(2-3):235--256.

\bibitem[Bubeck and Cesa-Bianchi, 2012]{BUBECK2012}
Bubeck, S. and Cesa-Bianchi, N. (2012).
\newblock Regret analysis of stochastic and nonstochastic multi-armed bandit
  problems.
\newblock {\em Foundations and Trends{\textregistered} in Machine Learning},
  5(1):1--122.

\bibitem[Chakravorty and Mahajan, 2014]{CHAKRAVORTY2014}
Chakravorty, J. and Mahajan, A. (2014).
\newblock Multi-armed bandits, gittins index, and its calculation.
\newblock {\em Methods and applications of statistics in clinical trials:
  Planning, analysis, and inferential methods}, 2(416-435):455.

\bibitem[Chaudhuri and Kalyanakrishnan, 2020]{CHAUDHURI2019}
Chaudhuri, A.~R. and Kalyanakrishnan, S. (2020).
\newblock Regret minimisation in multi-armed bandits using bounded arm memory.
\newblock In {\em The Thirty-Fourth {AAAI} Conference on Artificial
  Intelligence, {AAAI} 2020}, pages 10085--10092. {AAAI} Press.

\bibitem[Even-Dar et~al., 2002]{EVENDAR2002}
Even-Dar, E., Mannor, S., and Mansour, Y. (2002).
\newblock Pac bounds for multi-armed bandit and markov decision processes.
\newblock In {\em International Conference on Computational Learning Theory},
  pages 255--270. Springer.

\bibitem[Herschkorn et~al., 1996]{HERSCHKORN1996}
Herschkorn, S.~J., Pekoez, E., and Ross, S.~M. (1996).
\newblock Policies without memory for the infinite-armed bernoulli bandit under
  the average-reward criterion.
\newblock {\em Probability in the Engineering and Informational Sciences},
  10:21--28.

\bibitem[Karlin and Lei, 2015]{KarlinL15}
Karlin, A. and Lei, E. (2015).
\newblock On a competitive secretary problem.
\newblock In {\em Proceedings of the Twenty-Ninth {AAAI} Conference on
  Artificial Intelligence}, pages 944--950.

\bibitem[Lai and Robbins, 1985]{LAI1985}
Lai, T.~L. and Robbins, H. (1985).
\newblock Asymptotically efficient adaptive allocation rules.
\newblock {\em Advances in applied mathematics}, 6(1):4--22.

\bibitem[Liau et~al., 2018]{LIAU2018}
Liau, D., Song, Z., Price, E., and Yang, G. (2018).
\newblock Stochastic multi-armed bandits in constant space.
\newblock In {\em International Conference on Artificial Intelligence and
  Statistics}, pages 386--394.

\bibitem[Lu and Lu, 2011]{Lu2011}
Lu, C.-J. and Lu, W.-F. (2011).
\newblock Making online decisions with bounded memory.
\newblock In {\em International Conference on Algorithmic Learning Theory},
  pages 249--261. Springer.

\bibitem[Pek{\"o}z, 2003]{PEKOZ2003}
Pek{\"o}z, E.~A. (2003).
\newblock Some memoryless bandit policies.
\newblock {\em Journal of applied probability}, pages 250--256.

\bibitem[Robbins, 1952]{ROBBINS1952}
Robbins, H. (1952).
\newblock Some aspects of the sequential design of experiments.
\newblock {\em Bulletin of the American Mathematical Society}, 58(5):527--535.

\bibitem[Slivkins, 2019]{SLIVKINS2019}
Slivkins, A. (2019).
\newblock Introduction to multi-armed bandits.
\newblock {\em Found. Trends Mach. Learn.}, 12(1-2):1--286.

\bibitem[Thompson, 1933]{THOMPSON1933}
Thompson, W.~R. (1933).
\newblock On the likelihood that one unknown probability exceeds another in
  view of the evidence of two samples.
\newblock {\em Biometrika}, 25(3/4):285--294.

\bibitem[Tran-Thanh et~al., 2014a]{TRAN2014a}
Tran-Thanh, L., Stavrogiannis, L., Naroditskiy, V., Robu, V., Jennings, N.~R.,
  and Key, P. (2014a).
\newblock Efficient regret bounds for online bid optimisation in budget-limited
  sponsored search auctions.

\bibitem[Tran-Thanh et~al., 2014b]{TRAN2014b}
Tran-Thanh, L., Stein, S., Rogers, A., and Jennings, N.~R. (2014b).
\newblock Efficient crowdsourcing of unknown experts using bounded multi-armed
  bandits.
\newblock {\em Artificial Intelligence}, 214:89--111.

\end{thebibliography}
\newpage
\appendix
\section{Regret Minimization under Random Order Arrival}
\label{appendix: random regret}
\begin{thm}
In the \MAB setting, fix the number of arms $n$ and the time horizon $T$ . For any online \MAB algorithm, if we are allowed to store at most $m < n$ arms, then in the setting of random order arrival there exists an input instance such that $\mathbb{E}[R(T)] \geq \Omega(\frac{T^{2/3}}{m^{7/3}})$
\end{thm}
We consider $0$-$1$ rewards and the 2 input instances $\mathcal{I}_1,\mathcal{I}_2$ each containing $n$ arms, with parameter $\epsilon>0$ (where $\epsilon=\frac{1}{m^{1/3}T^{1/3}}$):
	\begin{align*}
	\mathcal{I}_1 = \left\{ \begin{array}{rcl}
	\mu_i & =(1+\epsilon)/2& \mbox{for~~}  i=1 \\ 
	\mu_i & = 1/2,~~~~~~~ &  \mbox{for~~}  i \neq 1
	\end{array}\right.
	\end{align*}

 \begin{align*}
	\mathcal{I}_2= \left\{ \begin{array}{rcl}
	\mu_i &= 1/2~~~~~~~  &\mbox{for~~} i\neq n\\
	\mu_i &= 1.~~~~~~~~~~ & \mbox{for~~}  i=n
	\end{array}\right.
	\end{align*}	
	\noindent In the above instances, $\mu_i$ denotes the expected reward of the $i^{th}$ arm in the input instance.\\
Let us fix a deterministic algorithm $\mathcal{A}$ that directly stores the first $m$ arms in the memory. We choose an input instance uniformly at random. Let this input instance be $\mathcal{I}'$. Then under random-order arrival setting one of the $n$ permutations of $\mathcal{I}'$ is chosen uniformly at random and is sent as the input stream to the Algorithm $\mathcal{A}$. Note that this equivalent to choosing a permutation $\mathcal{P}$ from $2n$ total distinct permutations of $\mathcal{I}_1$ and $\mathcal{I}_2$ uniformly at random and sending it to the Algorithm $\mathcal{A}$. Now assume that $m\leq n-1$. Let $\mathcal{I}_1'$ be the collection of distinct permutations of $\mathcal{I}_1$ such that the arm with expected reward of $(1+\epsilon)/2$ is in the first $m$ positions of the permutation.  Similarly, let $\mathcal{I}_2'$ be the collection of distinct permutations of $\mathcal{I}_2$ such that the arm with expected reward of $1$ is not in the first $m$ positions of the permutation. Clearly $|\mathcal{I}_1'|=m$ and $|\mathcal{I}_2'|=n-m$. 
Using arguments similar to Section \ref{sec:regretLower}, we can show that $\mathbb{E}[R(T)|\mathcal{P}\in\mathcal{I}_1'\cup\mathcal{I}_2']\geq \Omega\Big(\frac{T^{2/3}}{m^{7/3}}\Big)$. 
Hence, we have the following:
\begin{align*}
\mathbb{E}[R(T)] &\geq \mathbb{P}[\mathcal{P}\in\mathcal{I}_1'\cup\mathcal{I}_2']\cdot \mathbb{E}[R(T)|\mathcal{P}\in\mathcal{I}_1'\cup\mathcal{I}_2']\\
&\geq \frac{n-m+m}{2n}\cdot \Omega\Big(\frac{T^{2/3}}{m^{7/3}}\Big)\\
&\geq \Omega\Big(\frac{T^{2/3}}{m^{7/3}}\Big)
\end{align*}
The above result should hold for any randomized algorithm too as randomized algorithm are a distribution over deterministic algorithms.
\section{Important Inequalities}
\begin{lem}
(Hoeffding's inequality). Let $Z_1,\ldots,Z_n$ be independent bounded variables with $Z_i\in[0,1]$ for all $i\in[n]$. Then 
\begin{equation*}
\mathbb{P}\Big(\frac{1}{n}\sum_{i=1}^{n}(Z_i-\mathbb{E}[Z_i])\geq t)\leq e^{-2nt^2}\Big), \text{ and }
\end{equation*}
\begin{equation*}
\mathbb{P}\Big(\frac{1}{n}\sum_{i=1}^{n}(Z_i-\mathbb{E}[Z_i])\leq -t)\leq e^{-2nt^2}\Big),
\text{ for all } t\geq 0.
\end{equation*}
\end{lem}
\begin{lem}\label{comp:arm}
Let $\arm_1$ and $\arm_2$ be two different arms with means $\mean_1$ and $\mean_2$. Suppose $\mean_1 - \mean_2 \geq \theta$ and we sample each arm $\frac{K}{\theta^2}$ times to obtain empirical biases $\widehat{\mean}_1$ and $\widehat{\mean}_2$. Then,
\begin{equation*}
\mathbb{P}(\widehat{\mean}_1 \leq \widehat{\mean}_2)\leq 2\cdot e^{(-K/2)}
\end{equation*}
\end{lem}
\begin{proof}
\begin{align*}
\mathbb{P}(\widehat{\mean}_1 > \widehat{\mean}_2) & \geq \mathbb{P}\Big(\mean_1 - \frac{\theta}{2}< \widehat{\mean}_1 \text{ and }  \widehat{\mean}_2 < \mean_2 + \frac{\theta}{2}\Big)\\
&= \mathbb{P}\Big(\mean_1 - \frac{\theta}{2}< \widehat{\mean}_1\Big) \cdot \mathbb{P}\Big(\widehat{\mean}_2 < \mean_2 + \frac{\theta}{2}\Big) \\
&\geq \big(1- e^{-2\cdot\frac{K}{\theta^2}\cdot (\frac{\theta}{2})^2}\big)\cdot \big(1- e^{-2\cdot\frac{K}{\theta^2}\cdot (\frac{\theta}{2})^2}\big) \tag{Due to Hoeffding's Inequality}\\
& \geq (1-2\cdot e^{-K/2}).
\end{align*}
Hence, $\mathbb{P}(\widehat{\mean}_1 \leq \widehat{\mean}_2)\leq 2\cdot e^{-K/2}$.
\end{proof}

\begin{thm}[Berry-Esseen Theorem]
There exists a positive constant $C\leq 1$ such that if $X_1, X_2, \ldots, X_n$ are i.i.d.~random variables with $\mathbb{E}[X_i] = 0$, $\mathbb{E}[X_i^2] = \sigma^2 > 0$, and $\mathbb{E}[|X_i|^3] = \rho < \infty$, and if we define $$Y_n = \frac{X_1 + X_2 + \ldots + X_n}{n}$$ to be the sample mean, with $F_n$ being the cumulative distribution function of $\frac{Y_n\sqrt{n}}{\sigma}$, and $\Phi$ be the cumulative distribution function of the standard normal distribution $\mathcal{N}(0,1)$, then for all $x$ and $n$,
$$|F_n(x) - \Phi(x)| \leq \frac{C\rho}{\sigma^3\sqrt{n}}.$$
\end{thm}
\section{Omitted Proofs from Section \ref{section:round algo}}
\label{appendix: omitted proofs}

\RoundAlgoLevelSuboptimality*
\begin{proof}
	Let us consider some level $\ell$ and a set of arms $\arm_{\ell_1},\arm_{\ell_2},\ldots,\arm_{\ell c_\ell}$, which increases the counter $C_\ell$ from $0$ to $c_\ell$. Denote the best arm among them as $\arm'$ and let $\arm^*_\ell$ be  the empirically best arm seen so far since the arrival of $\arm_{\ell_1}$.  Our algorithm works as follows: once an arm arrives at level $\ell$, it is pulled $s_\ell$ number of times. Next, we compare its empirical mean with that of $\arm^*_\ell$, which was computed when the arm corresponding to $\arm^*_\ell$ arrived at level $\ell$. The arm with the greater empirical mean is maintained as $\arm^*_\ell$ and its empirical mean is stored for future comparisons. Once $C_\ell = c_\ell$, $\arm^*_\ell$ is sent to level $\ell + 1$. Note that if $0<C_\ell< c_\ell$, and the condition in Step \ref{st:whilee} is not satisfied then we will sample $\arm^*_\ell$ in the Step \ref{special:st}.  Note that this is equivalent to simultaneously pulling each of the $c_\ell$ arms $s_\ell$ number of times, and then sending the empirically best arm to level $\ell+1$.
	
	Using Lemma \ref{comp:arm}, we can show that at any level $\ell$, if two arms have reward gap $\mean_i - \mean_j \geq \varepsilon_\ell$ then $\mathbb{P}[\widehat{\mean}_i < \widehat{\mean}_j] \leq \frac{\delta}{2^{\ell+1}\cdot c_\ell}$. Since at most $c_\ell$ arms arrive at level $\ell$ before the counter $C_\ell$ reaches $c_\ell$ or the condition in Step \ref{st:whilee} is not satisfied, taking union bound we have that the probability that an arm with reward gap $\geq \varepsilon_\ell$ with respect to $\arm'_\ell$ is either sent to level $\ell+1$ or is sampled in the Step \ref{special:st} is at most $\frac{\delta}{2^{\ell+1}\cdot c_\ell}\cdot c_\ell=\frac{\delta}{2^{\ell+1}}$.
\end{proof}

\RoundAlgoLevelSuboptimalityTwo*
\begin{proof}
	Using Lemma \ref{comp:arm}, we can show that among the arms that are being considered here, if two arms have reward gap $\mean_i - \mean_j \geq \varepsilon_r$ then $\mathbb{P}[\widehat{\mean}_i < \widehat{\mean}_j] \leq \frac{\delta}{2^{r+1}\cdot c_r}$. As $c_r=n$, by taking union bound we get that with probability at least $1 - \frac{\delta}{2^{r+1}}$, an arm with reward gap at most $\varepsilon_r$ from $\arm'_r$ is returned by the Algorithm.
\end{proof}	
\section{Adversarial Example for  \cite{ASSADI2020}}
\label{subsec:adversarial example}
%
In this section, we show an adversarial example to show that the algorithm of \cite{ASSADI2020} is not $(\varepsilon, \delta)$-PAC.
First, we state the algorithm of \cite{ASSADI2020} (Algorithm \ref{algo:AssadiOriginal}) here for completeness. The algorithm takes as input $n \in \mathbb{N}$ arms arriving in a stream in an arbitrary order, an approximation parameter $\varepsilon \in [0,1/2)$, and the confidence parameter $\delta \in (0,1)$. We first define some notation that is used throughout this work, which is consistent with the notation used in \cite{ASSADI2020}.
\begin{align*}
	&\{r_\ell\}_{\ell=1}^\infty:\hspace{0.3cm} r_1 = 4;~~ r_\ell = 2^{r_\ell}; \tag{intermediate parameters used to define $s_\ell$ below}\\
	&\varepsilon_\ell = \varepsilon/({10\cdot 2^{\ell - 1}}); &\tag{intermediate estimate of gap parameter}\\
	&\beta_\ell = {1}/{\varepsilon_\ell^2};  \\
	&\{s_\ell\}_{\ell=1}^\infty:\hspace{0.3cm} s_\ell = 4\beta_\ell\big(\ln(1/\delta) + 3r_\ell \big); \tag{number of samples per arm in level $\ell$}\\
	&\{c_\ell\}_{\ell=1}^{\ceil{\log^* n} + 1}:\hspace{0.3cm} c_1 = 2^{r_1};~~ c_\ell = {2^{r_\ell}}/{2^{\ell - 1}};\\& \text{(the number of arms processed in level $\ell$ before}\\ &\text{sending $\arm^*_\ell$ to level $\ell+1$)}
\end{align*}

For Algorithm \ref{algo:AssadiOriginal} to be $(\varepsilon,\delta)$-PAC, it has to find an $\varepsilon$-best arm with probability at least $1-\delta$. We next provide a formal argument for why Algorithm \ref{algo:AssadiOriginal} is not an $(\varepsilon,\delta)$-PAC algorithm by providing a counterexample. 
\begin{algorithm}[ht!]
	\caption{}
	\label{algo:AssadiOriginal}
	\begin{algorithmic}[1]
		\STATE $\{r_\ell\}^{\infty}_{\ell=1}:r_1:=4,r_{\ell+1}=2^{r_\ell}$;
		\STATE $\varepsilon_\ell=\frac{\varepsilon}{10\cdot 2^{\ell-1}}$;
		\STATE $\beta_\ell=\frac{1}{\varepsilon^2_\ell}$;
		\STATE $s_\ell=4\beta_\ell\big(\ln(\frac{1}{\delta})+3r_\ell\big)$;
		\STATE $c_1=2^{r_1},$ $c_\ell=\frac{2^{r_\ell}}{2^{\ell-1}}(\ell \geq 2)$;
		\STATE Counters: $C_1,C_2,\ldots,C_t$ initialized to 0 where $t=\lceil \log^*(n)\rceil+1$.
		\STATE Stored arms: $\arm^*_1,\arm^*_2,\ldots,\arm^*_t$ the most biased arm of $\ell$-th level.
		\STATE Stored empirical means:$p_1,p_2,\ldots,p_t$ the highest empirical mean of $\ell$-th level.
		\WHILE {A new arm $\arm_i$ arrives in the stream}
		\STATE Read $\arm_i$ to memory
		\STATE \textbf{\underline{Aggressive Selective Promotion:}} Starting from level $\ell=1$:
		\STATE \label{next level} Sample \textbf{both} $\arm_i$ and $\arm^*_\ell$ for $s_\ell$ times.
		\STATE Drop $\arm_i$ if $\hat{p}_{\arm_i} < \hat{p}_{\arm^*_\ell}$, otherwise replace $\arm^*_\ell$ with $\arm_i$.
		\STATE Increase $C_\ell$ by 1.
		\STATE If $C_\ell=c_\ell$, make $C_\ell$ equal to 0, send $\arm^*_\ell$ to the next level by calling Line \ref{next level} with $(\ell=\ell+1)$.
		\ENDWHILE
		\STATE Return $\arm^*_t$ as the selected most bias arm.
	\end{algorithmic}
\end{algorithm}

Let the arms arrive in the stream be $\arm_1,\ldots,\arm_n$, where $n>(c_1\cdot c_2)$. 
Let $\arm_i$ be the $i^{th}$ arm to arrive in the stream and has a mean $p_i$, where $i\in[n]$. 
For all $i>c_1\cdot c_2$, let $p_i=0$. 
For all $i\leq c_1\cdot c_2$, let $p_i=\frac{1}{2}-(\lceil \frac{i}{c_1}-1\rceil)\cdot\frac{\varepsilon}{c_2-2}$. 
Let $\arm_{k_1},\arm_{k_2},\ldots,\arm_{k_{c_2}}$ be the first $c_2$ arms which arrive at level $2$ (note that all the arms which arrive at level $2$ after $\arm_{k_{c_2}}$ will have a mean of 0). Let $\arm^*_2$ be the most biased arm (based on the sampling) at the end of Aggressive Selection Promotion step for level $\ell=2$ for  $\arm_{k_{c_2}}$. 
Now $C_2 = c_2$ after the arrival of $\arm_{k_{c_2}}$. 
Thus $\arm^*_2$ will be sent to level $3$.
As all the following remaining arms have lesser means, at the end, the algorithm finally returns an arm with mean less than or equal to the mean of $\arm^*_2$. 
Note that $\forall i\in [c_2]$, all the arms in the set $\{\arm_{(i-1)\cdot c_1+1},\ldots,\arm_{i\cdot c_1}\}$ have the same mean and one among them is sent as $\arm_{k_{i}}$ to level $2$. Therefore $p_{k_i} = p_{k_1}-(i-1)\frac{\varepsilon}{b}$, where $b = 2^{15}-2$ and $p_{k_1}=\frac{1}{2}$. So for any $i \in [c_{2}-1]$, $p_{k_i} - p_{k_{i+1}} = \frac{\varepsilon}{b}$. 
We will show that with probability $>\delta$, we send $\arm_{k_{c_2}}$ to level $3$, and $\frac{1}{2}-p_{k_{c_2}}>\varepsilon$. 
 
For $i\in[c_2]$, let $Y_i^t$ denote the reward when we sample the arm $\arm_{k_i}$ for the $t$-th time. We assume that $Y_i^t \sim \text{Bern}(p_{k_i})$ and $\mathrm{Var}[Y_i^t]=p_{k_i}(1-p_{k_i})$ (Note that this is a reasonable assumption as the Algorithm  \ref{algo:AssadiOriginal} should work for any distribution). For $i \in [c_2 - 1]$, let $Z_i^t = Y_i^t - Y_{i+1}^t$. Clearly, $\mu_i:=\mathbb{E}[Z_i^t] = \frac{\varepsilon}{b}$. Let $\sigma_i^2:=\mathrm{Var}[Z_{i}^t]$. Let us assume that $\varepsilon<\frac{1}{5}$ (Later we will choose $\varepsilon$ in such a way so that this condition is satisfied). In this case, $\sigma_i^2=\mathrm{Var}[Y_i^t]+\mathrm{Var}[Y_{i+1}^t]>2(p_{k_{c_2}})(1-p_{k_{c_2}})>\frac{2}{5}$. Let $Z_i = Z_i^1 + Z_i^2 + \ldots + Z_i^{s_2}$. Note that, if every arm from the set $\arm_{k_1},\arm_{k_2},\ldots,\arm_{k_{c_2}}$ when it arrives in the level $2$ beats $\arm^*_2$ in the challenge, then $\arm_{k_{c_2}}$ will be sent to level $3$. 
Thus, $\{Z_i < 0, \forall i\in[c_2-1]\} \subseteq \{\arm_{k_{c_2}} ~\text{is sent to level}~ 3\}$.

Assuming that $\delta$ and $\varepsilon$ are very small (which we will choose appropriately to bound the error), we approximate (using the central limit theorem) the distribution of $Z_i$ using the normal distribution $\mathcal{N}(s_2\mu_i,s_2\sigma_i^2)$.

\begin{align*}
\mathbb{P}[Z_i<0]  &= \mathbb{P}[Z_i > 2s_2\mu_i]\\
& = 1-\frac{1}{2}\bigg(1+\text{erf}\Big(\frac{s_2\mu_i}{\sqrt{2s_2 \sigma_i^2}}\Big)\bigg)\\
&= \frac{\text{erfc}\Big(\frac{s_2\mu_i}{\sqrt{2s_2 \sigma_i^2}}\Big)}{2}\\
& \geq \frac{\text{erfc}\Big(s_2\mu_i \cdot \sqrt{\frac{5}{4s_2}}\Big)}{2} \tag{Since, $\sigma_i^2 \geq \frac{2}{5}$ and erfc$(x)$ is decreasing in $x$}\\
& = \frac{\text{erfc}\Big(\frac{20\sqrt{5}}{b}\sqrt{\ln\big(\frac{e^{3r_2}}{\delta}\big)}\Big)}{2} \tag{Substituting $s_2 = 4\beta_2\big(\ln(\frac{1}{\delta})+3r_2\big)$}\\
&\geq \frac{\sqrt{\gamma-1}}{2}e^{-\frac{2000\cdot \gamma}{b^2}\ln(\frac{e^{3r_2}}{\delta})}\tag{Since $\text{erfc}(x)\geq {(\gamma-1)}^{1/2}e^{-\gamma x^2}, \forall x\geq 0, \text{where }\gamma:=\sqrt{{2e}/{\pi}}$}\\
&= \frac{\sqrt{\gamma-1}}{2}\Big(\frac{\delta}{e^{3r_2}}\Big)^{\frac{2000\cdot \gamma}{b^2}}.\\ 
\end{align*}

Thus, we can lower bound the probability that $\arm_{k_{c_2}}$ is sent to level $3$ as follows:
\begin{align*}
\mathbb{P}[\arm_{k_{c_2}} ~\text{is sent to level}~ 3]
 &\geq \mathbb{P}[Z_i < 0, \forall i\in[c_2-1]]\\
& = \prod_{i\in [c_2 -1]}\mathbb{P}[Z_i < 0] \tag{Since, the arm pulls are independent}\\
& \geq \Big( \frac{\sqrt{\gamma-1}}{2}\Big(\frac{\delta}{e^{3r_2}}\Big)^{\frac{2000\cdot \gamma}{b^2}}\Big)^{c_2-1}\\
& = \frac{\delta^{\frac{2000\cdot\gamma\cdot(c_2-1)}{b^2}}}{K}, \text{ where } K= \Big(\frac{2\cdot e^{\frac{3r_2\cdot 2000\cdot \gamma}{b^2}}}{\sqrt{\gamma-1}}\Big)^{c_2-1}.
\end{align*}

Consider that function $f(x)=\frac{e^{x\left(1-\frac{2000\cdot\gamma\cdot(c_2-1)}{b^2}\right)}}{K}$. Since $f(x)$ is an increasing and convex function, there is a constant $c$ such that $f(c)>2$. This implies that for $\delta=e^{-c}$ we have the following:
\begin{align*}
\mathbb{P}[\arm_{k_{c_2}} ~\text{is sent to level}~ 3] & \geq \frac{\delta^{\frac{2000\cdot\gamma\cdot(c_2-1)}{b^2}}}{K}\\
& = f(c)e^{-c}\\
& >2\delta.
\end{align*}

Now, we bound the error in calculation of the above probability. Using the Berry-Esseen theorem, the error $\epsilon_i$ of calculating $\mathbb{P}[Z_i<0]$ is upper bounded by $\frac{C \rho}{\sigma_i^3\sqrt{s_2}}\leq \frac{ \varepsilon}{\sqrt{\ln(\frac{1}{\delta})}}$, where $C\leq 1$, $\rho=\mathbb{E}[|Z_i^t-\mean_i|^3]\leq 8 \text{ (as } |Z_i^t-\mean_i|\leq 2)$ and $\sigma_i^2=\mathrm{Var}[Z_i^t-\mean_i]=\mathrm{Var}[Z_i^t]$. Also we assumed $\mathrm{Var}[Z_i^t]>\frac{2}{5}$ (we will choose $\varepsilon$ in such a way that this is satisfied). If we choose $\varepsilon$ such that $\varepsilon < \frac{\delta\sqrt{\ln(\frac{1}{\delta})}}{c_2}$ and $\varepsilon<1/5$, then $\epsilon_i<\frac{\delta}{c_2}$. Hence, we can conclude that $\mathbb{P}[\arm_{k_{c_2}} ~\text{is sent to level}~ 3]>2\delta-\sum_{i=1}^{c_2-1}\epsilon_i > 2\delta-\delta=\delta$.

As $\frac{1}{2}-p_{k_{c_2}}=\Big(\frac{c_2-1}{c_2-2}\Big)\cdot\varepsilon$, we can conclude that with probability $>\delta$, the Algorithm \ref{algo:AssadiOriginal} returns an arm with reward gap $> \varepsilon$.

\section{Lower bound for r-round adaptive streaming algorithm}\label{r-lowerrr-bnd}
In this section, we use the following lower bound for $r$-round adaptive offline algorithm model defined in \cite{AGARWAL2017} to provide a lower bound on the sample complexity for any $r$-round adaptive streaming algorithm. 

\begin{lem}\label{agar:thm4}
(\cite{AGARWAL2017}). For any parameter $\Delta\in (0,1/2)$ and any integer $n,k\geq 1$, there exists a distribution $\mathcal{D}$ on input instances of the $k$ most biased coins problem with $n$ coins and gap parameter $\Delta_k=\Delta$ such that for any integer $r\leq 1$, any $r$-round algorithm that finds the $k$ most biased coins in the instances sampled from $\mathcal{D}$ with probability at least $3/4$ has a sample complexity $\Omega(\frac{n}{\varepsilon^2\cdot r^4}\cdot \ilog^{(r)}(n/k))$
\end{lem}

Any $r$-round adaptive offline algorithm is same as $r$-round adaptive streaming algorithm except for the following two points:
\begin{enumerate}
\item In an $r$-round adaptive offline algorithm all the arms can be simultaneously stored in the memory, whereas arm-memory is usually bounded in an $r$-round adaptive streaming algorithm.
\item In an $r$-round adaptive offline algorithm, in any round $j$, all the arms are sampled simultaneously. On the other hand, in an $r$-round adaptive streaming algorithm all the arms in the round $j$ need not be sampled simultaneously, and they can also be sampled one after the other.
\end{enumerate}
Any $r$-round adaptive streaming algorithm can be replicated by an $r$-round adaptive offline algorithm such that the worst-case sample complexity is the same in both the algorithms. Next, we present the following lemma.
\begin{lem}
For any approximation parameter $\varepsilon\in (0,1/2)$ and any integer $n\geq 1$, there exists a distribution $\mathcal{D}$ on input instances of the best-arm identification problem with $n$ arms  such that for any integer $r\geq 1$, any $r$-round adaptive streaming algorithm that finds the $\varepsilon$-best arm in the instances sampled from $\mathcal{D}$ with probability at least $3/4$ has a sample complexity $\Omega(\frac{n}{\varepsilon^2\cdot r^4}\cdot \ilog^{(r)}(n))$
\end{lem}
\begin{proof}
We present the proof idea. Let $\mathcal{A}$ be an $r$-round adaptive streaming algorithm with the lowest worst-case sample complexity which finds the $\varepsilon$-best arm with probability at least $3/4$. Now we replicate the algorithm $\mathcal{A}$ using an $r$-round offline algorithm $\mathcal{B}$ for best-arm identification such that its worst-case sample complexity is equal to that of $\mathcal{A}$. Due to Lemma \ref{agar:thm4}, the worst-case sample complexity of $\mathcal{B}$ is $\Omega(\frac{n}{\varepsilon^2\cdot r^4}\cdot \ilog^{(r)}(n))$. Hence, the worst-case sample complexity of $\mathcal{A}$ is $\Omega(\frac{n}{\varepsilon^2\cdot r^4}\cdot \ilog^{(r)}(n))$.
\end{proof}

\section{Adversarial Example for Constant Arm-Memory Algorithm based on \cite{ASSADI2020}}\label{assadi_type}

In this section, we show an adversarial example for Constant Arm-memory Algorithm based on \cite{ASSADI2020} when the parameter $\Delta$ is unknown where $\Delta$ is the gap between the best arm and the second best arm.  For completeness, we present this algorithm below (Algorithm \ref{algo:AssadiType}).

\begin{algorithm}[ht!]
	\caption{}
	\label{algo:AssadiType}
	\begin{algorithmic}[1]
		\STATE $\{r_\ell\}^{\infty}_{\ell=1}:r_\ell=3^\ell$;
		\STATE $\{s_\ell\}^{\infty}_{\ell=1}:s_\ell=\frac{2}{\varepsilon^2}\cdot\ln\left(\frac{1}{\delta}\right)\cdot r_\ell$;
		\STATE $b:=\frac{2}{\epsilon^2}\cdot C \cdot \ln\left(\frac{1}{\delta}\right)+s_1$;
		\STATE Let \textbf{king} be the first available arm and set its budget $\phi:=\phi(king)=0$.
		\WHILE {A new arm $\arm_i$ arrives in the stream}\label{assadi-type:st1}
		\STATE Increase the budget $\phi(\textbf{king})$ by $b$. 
		\STATE \textbf{\underline{Challenge subroutine:}} For level $\ell=1$ to $+\infty$ :
		\STATE If $\phi(king)<s_\ell$: we declare $\textbf{king}$ defeated , make $\arm_i$ the king, initialize its budget to 0 and go to Step \ref{assadi-type:st1}.
		\STATE Otherwise, we decrease $\phi(\textbf{king})$ by $s_\ell$ and sample both king and $\arm_i$ for $s_\ell$ times.
		\STATE Let $\widehat{\mean}_{king}$ and $\widehat{\mean}_i$ denote the empirical biases of king and $\arm_i$ in this trial.
		\STATE If $\widehat{\mean}_{king} > \widehat{\mean}_i$, we declare $king$ winner and go to the next arm in the stream; otherwise we go to the next level of challenge (increment $\ell$ by one).
		\ENDWHILE
		\STATE Return $\textbf{king}$ as the selected most bias arm.
	\end{algorithmic}
\end{algorithm}

Let $\arm_{1},\arm_{2},\ldots,\arm_{n}$ be the stream of arms which arrive. For each $i \in [n]$, $\mean_{i}$ is the expected reward of arm $\arm_{i}$. Define the input stream such that $\mean_{i} = \mean_{1}-(i-1)\frac{\varepsilon}{n-2}$. Note that for any $i \in [n-1]$, $\mean_{i} - \mean_{i+1} = \frac{\varepsilon}{n-2}$. Let $\mean_1=\frac{1}{2}$.

We now show that there exist $\delta,\varepsilon>0$ such that with probability $>\delta$, an arm with reward gap $> \varepsilon$ is returned by the Algorithm. Let $k$ be the maximum value of $\ell$ such that $\sum_{i=2}^{\ell}3^i\leq C$. Let $\forall i\in[n]$, $Y_{i,t}^\ell$ denote the reward when we sample the arm $\arm_i$ for the $t^{th}$ time at level $\ell$. Then, $\mathrm{Var}[Y_{i,t}^\ell]= \mean_i(1 - \mean_i)$. Let $Z_{i,t}^\ell=Y_{i,t}^\ell-Y_{i+1,t}^\ell$. Clearly $p_i:=\mathbb{E}[Z_{i,t}^\ell]=\frac{\varepsilon}{n-2}$. Let $\sigma_i^2:=\mathrm{Var}[Z_{i,t}^\ell]$. Let us assume that $\varepsilon<\frac{1}{5}$ and $n>>100$ (We will later choose $\varepsilon$ and $n$ in a way so that this condition is satisfied). In this case, $\sigma_i^2=\mathrm{Var}[Y_{i,t}^\ell]+\mathrm{Var}[Y_{i+1,t}^\ell]>2(\mean_n)(1-\mean_n)>\frac{2}{5}$. Let $Z_i^\ell=Z_{i,1}^\ell+Z_{i,2}^\ell+\ldots+Z_{i,{s_\ell}}^\ell$. Assuming that $\delta,\varepsilon$ are very small (which we will choose appropriately to bound the error) we approximate (using Central Limit Theorem) the distribution of $Z_i^\ell$ using the normal distribution $\mathcal{N}(s_\ell p_i,s_\ell\sigma_i^2)$.

\begin{align*}
\mathbb{P}[Z_i^\ell<0]&=\mathbb{P}[Z_i^\ell > 2s_\ell p_i]\\
&=1-\frac{1}{2}\bigg[1+\text{erf}\bigg(\frac{s_\ell p_i}{\sqrt{2s_\ell \sigma_i^2}}\bigg)\bigg]\\
&=\frac{\text{erfc}\Big(\frac{s_\ell p_i}{\sqrt{2s_\ell\sigma_i^2}}\Big)}{2}\\
& \geq \frac{\text{erfc}(\frac{s_\ell p_i\sqrt{5}}{2\sqrt{s_\ell}})}{2}\tag{ Since, erfc(.) is a decreasing function and $\sigma_i^2\ge 2/5$}\\
& \geq \frac{\text{erfc}\Big(\frac{\sqrt{5}\cdot 3^{\ell/2}}{n-2}\sqrt{\ln(\frac{1}{\delta})}\Big)}{2}\tag{ substituting value of $s_\ell$, erfc(.) is a decreasing function}\\
&\geq \frac{\sqrt{\gamma-1}}{2}e^{-\frac{5\cdot 3^{\ell}\cdot \gamma}{(n-2)^2}\ln(\frac{1}{\delta})} \tag{ as $\text{erfc}(x)\geq ({\sqrt{\gamma-1}})e^{-\gamma x^2}, \forall x\geq 0,\text{where } \gamma=\sqrt{{2e}/{\pi}}$ }\\
&= \frac{\sqrt{\gamma-1}}{2}\delta^{\frac{5\cdot 3^{\ell}\cdot \gamma}{(n-2)^2}}.
\end{align*}
Hence, we now have 
\begin{align*}
\mathbb{P}[\arm_{i+1}\text{ becomes king by defeating }\arm_{i}] &\geq \mathbb{P}[\forall \ell\in[k], Z_i^\ell<0]\\
& \geq \left(\frac{\sqrt{\gamma-1}}{2}\delta^{\frac{5\cdot 3^{k}\cdot \gamma}{(n-2)^2}}\right)^{k}\\
& = \frac{\delta^{\frac{5\cdot 3^{k}\cdot k \cdot \gamma}{(n-2)^2}}}{K} \tag{ where $K= \left(\frac{2}{\sqrt{\gamma-1}}\right)^{k}$}
\end{align*}

\noindent Now we choose $n$ such that $\frac{(n-2)^2}{(n-1)}>>5\cdot 3^{k}\cdot k \cdot \gamma$. 

Consider the function $f(x)=\frac{e^{x\left(1-\frac{5\cdot 3^{k}\cdot k \cdot \gamma\cdot(n-1)}{(n-2)^2}\right)}}{K^{n-1}}$. Since $f(x)$ is an increasing and convex function, there is a constant $c$ such that $f(c)>2$. This implies that for $\delta=e^{-c}$ we have the following:
\begin{align*}
\mathbb{P}[\arm_{n}\text{ is returned as king by the algorithm}]
&  \geq \mathbb{P}[\forall i\in[n-1], \arm_{i+1}\text{ becomes king by defeating }\arm_{i}]  \\
& \geq  \left(\frac{\delta^{\frac{5\cdot 3^{k}\cdot k \cdot \gamma}{(n-2)^2}}}{K}\right)^{n-1}\\
& = \frac{e^{c\left(1-\frac{5\cdot 3^{k}\cdot k \cdot \gamma\cdot (n-1)}{(n-2)^2}\right)}}{K^{n-1}}\cdot \delta\\
& = f(c)\cdot \delta\\
& >2\delta.
\end{align*}

Now we bound the error in calculation of the above probability. By Berry-Esseen theorem, the error $\epsilon_i^\ell$ of calculating $\mathbb{P}[Z_i^\ell<0]$ is upper bounded by $\frac{C \rho}{\sigma_i^3\sqrt{s_1}}\leq \frac{16 \varepsilon}{\sqrt{\ln(\frac{1}{\delta})}}$. Here $C\leq 1$, $\rho=\mathbb{E}[|Z_{i,t}^\ell-p_i|^3]\leq 8( \text{as } |Z_{i,t}^\ell-p_i|\leq 2)$ and $\sigma_i^2=\mathrm{Var}[Z_{i,t}^\ell-p_i]=\mathrm{Var}[Z_{i,t}^\ell]$. Also we assumed $\mathrm{Var}[Z_{i,t}^\ell]>\frac{2}{5}$ (we will choose $\varepsilon$ in such a way that this is satisfied). If we choose $\varepsilon$ such that it is less than $\frac{\delta\sqrt{\ln(\frac{1}{\delta})}}{16\cdot k\cdot n}$ and it is also less than $\frac{1}{5}$, then $\epsilon_i^\ell<\frac{\delta}{k\cdot n}$. Therefore $\mathbb{P}[\arm_{n}\text{ is returned as king by the algorithm}]>2\delta-\sum_{i=1}^{n-1}\sum_{\ell=1}^{k}\epsilon_i^\ell>2\delta-\delta=\delta$. As $\mean_{1} - \mean_n>\varepsilon$, we can conclude that there exists $\delta,\varepsilon>0$ such that with probability $>\delta$, a king with reward gap $> \varepsilon$ is returned by the algorithm. 

We also have experimental evidence for this result. We ran the experiment on problem instances of this adversial type with number of arms $= 5500001$,  mean of the best arm $=1/2$, $\varepsilon=1/10$. 
We ran ten independent experiments with different realizations due to different values from sampling.  Each time the algorithm didn't return $\varepsilon$-best arm.


\section{Algorithm with constant arm-memory}
\label{randomorder}
We now tweak the Algorithm in Section \ref{assadi_type} and show that our new proposed Algorithm deals with the following:
\begin{itemize}
\item works well both theoretically and experimentally on the counter example for the algorithm mentioned in Section \ref{assadi_type}.
\item works well theoretically on random order arrival for well known distributions 
\item works well experimentally on any randomly generated input from some well known distributions.
\end{itemize}

\begin{lem}
\label{Lem: random_order_1}
Let $\arm_1$ and $\arm_2$ be two different arms with biases $\mean_1$ and $\mean_2$. Suppose $\mean_1 - \mean_2 \geq 0.5\varepsilon$ and we sample each arm $s_\ell$ times to obtain empirical biases $\widehat{\mean}_1$ and $\widehat{\mean}_2$. Then,
\begin{equation*}
\mathbb{P}(\widehat{\mean}_1 \leq \widehat{\mean}_2 + 0.495\varepsilon)\leq 2\cdot e^{(-\ln(\frac{4}{\delta})\cdot r_\ell)}.
\end{equation*}
\end{lem}
\begin{proof}
\begin{align*}
\mathbb{P}(\widehat{\mean}_1 > \widehat{\mean}_2 + 0.495\varepsilon)
& \geq \mathbb{P}\Big(\mean_1 - \frac{\varepsilon}{400} < \widehat{\mean}_1  \text{ and }  \widehat{\mean}_2 < \mean_2 + \frac{\varepsilon}{400}  \Big)\\
&= \mathbb{P}\Big(\mean_1-\frac{\varepsilon}{400}< \widehat{\mean}_1\Big) \cdot \mathbb{P}\Big(\widehat{\mean}_2 < \mean_2 + \frac{\varepsilon}{400}\Big) \\
&\geq (1- e^{-\ln(\frac{4}{\delta})\cdot r_\ell})\cdot(1- e^{-\ln(\frac{4}{\delta})\cdot r_\ell}) \tag{ Due to Hoeffding's Inequality}\\
& \geq (1- 2\cdot e^{-\ln(\frac{4}{\delta})\cdot r_\ell}).
\end{align*}
Hence, we have
$\mathbb{P}(\widehat{\mean}_1 \leq \widehat{\mean}_2 + 0.495\varepsilon) \leq 2\cdot e^{-\ln(\frac{4}{\delta})\cdot r_\ell}$.
\end{proof}
\begin{lem}\label{0.49lem}
Let $\arm_1$ and $\arm_2$ be two different arms with biases $\mean_1$ and $\mean_2$. Suppose $\mean_1 - \mean_2 \leq 0.49\varepsilon$ and we sample each arm $s_\ell$ times to obtain empirical biases $\widehat{\mean}_1$ and $\widehat{\mean}_2$. Then,
\begin{equation*}
\mathbb{P}(\widehat{\mean}_1\geq\widehat{\mean}_2+0.495\varepsilon)\leq 2\cdot e^{-\ln(\frac{4}{\delta})\cdot r_\ell}.
\end{equation*}
\end{lem}
\begin{proof}
\begin{align*}
\mathbb{P}(\widehat{\mean}_1 < \widehat{\mean}_2 + 0.495\varepsilon)
& \geq \mathbb{P}\Big(\widehat{\mean}_1 < \mean_1 + \frac{\varepsilon}{400} \text{ and } \mean_2 - \frac{\varepsilon}{400}< \widehat{\mean}_2\Big)\\
&= \mathbb{P}\Big( \widehat{\mean}_1 < \mean_1 + \frac{\varepsilon}{400}\Big)\cdot \mathbb{P}\Big(\mean_2 - \frac{\varepsilon}{400}< \widehat{\mean}_2\Big) \\
&\geq (1- e^{-\ln(\frac{4}{\delta})\cdot r_\ell})\cdot(1- e^{-\ln(\frac{4}{\delta})\cdot r_\ell}) \tag{ Due to Hoeffding's Inequality}\\
& \geq (1- 2\cdot e^{-\ln(\frac{4}{\delta})\cdot r_\ell}).
\end{align*}
Hence, we have
$\mathbb{P}(\widehat{\mean}_1 \geq \widehat{\mean}_2 + 0.495\varepsilon)\leq 2\cdot e^{-\ln(\frac{4}{\delta})\cdot r_\ell}
$.
\end{proof}
\begin{lem}\label{king:cond1}
In a challenge subroutine, if $\mean_i - \mean_{king} \geq 0.5\varepsilon$, then the probability that $\arm_i$ does not become the king is at most $\frac{\delta}{8}$. 
\end{lem}
\begin{proof}
\begin{align*}
\mathbb{P}(\arm_i\text{ loses to king })
&\leq\sum_{\ell=1}^\infty \mathbb{P}(\arm_i\text{ loses to king }\text{at level $\ell|\arm_i$ has not lost until }\ell-1)\\
&\leq \sum_{\ell=1}^\infty 2\cdot e^{-\ln(\frac{4}{\delta})\cdot r_\ell} \tag{From Lemma \ref{Lem: random_order_1}}\\
& < (\delta/2) \cdot \sum_{\ell=1}^\infty e^{-3^\ell}\\
& < \delta/8.
\end{align*}
Since the budget is finite, king will lose to $\arm_i$ with probability at least $(1-\delta/8)$ in finite time.
\end{proof}
The next lemma is an adaptation of Lemma 3.3 in (Assadi et al. \citeyear{ASSADI2020}).
\begin{lem}\label{assadi:kinglem}
In Algorithm \ref{algo:new_algo}, if any incoming arm does not lose to the king (denoted $\arm_{king}$)  at a level $\ell$ with probability at most $2\cdot e^{-\ln(\frac{1}{\delta'})\cdot 3^\ell}$, then the probability that $\arm_{king}$ loses to any incoming arm is at most $\delta'/2$. 
\end{lem}
\begin{lem}\label{king:notlose}
Let $\mean_{king}$ be the bias of the current king ($\arm_{king}$). If the future arms in the stream don't have a bias in the range 
$(\mean_{king} + 0.49\epsilon, \mean^*]$ where $\mean^*$ is the bias of the most-biased arm, then the probability that the king is ever defeated is at most $\delta/8$.
\end{lem}
\begin{proof}
The lemma follow from Lemma \ref{assadi:kinglem} and Lemma \ref{0.49lem} by substituting $\delta'=\delta/4$.
\end{proof}
\begin{lem}\label{king:cond2}
Let the $\mean^*$ be the bias of the most-biased arm. If a $\arm_i$ with bias $\mean_i \in [\mean^* - 0.49\varepsilon,\mean^*]$ becomes the king, then the probability that $arm_i$ is ever defeated as a king is at most $\delta/8$.	
\end{lem}
\begin{proof} 
This follows directly from Lemma \ref{king:notlose}.
\end{proof}
\begin{cor}
If the input stream is the counter example for the algorithm mentioned in the Section \ref{assadi_type}, then the probability that Algorithm \vis{\ref{algo:new_algo}} returns a non-$\varepsilon$-best arm is at most $\delta/8$.
\end{cor}
\subsection{Random Order Arrival}
Let the number of arms $n$ in the input set of arms be very large such that the distribution of means of the input set of arms becomes sort of continuous. Let this distribution have a P.D.F $f(x)$. Let $\mean^*$ be the mean of the best arm. 
We choose a random permutation of our input and send it as an input stream to Algorithm \ref{algo:new_algo}.
\begin{lem}
Under random order arrival, probability that Algorithm \ref{algo:new_algo} returns an $\varepsilon$-best arm is at least 
\begin{align*}
\inf_{p'\in[\mean^*-0.99\varepsilon,\mean^*-0.5\varepsilon]}\Bigg\{(1-\delta)\cdot
\bigg(\frac{\int_{\min\{p'+0.01\varepsilon,\mean^*-0.5\varepsilon\}}^{\mean^*-0.5\varepsilon} f(x) \frac{\int_{x+0.5\varepsilon}^{\mean^*} f(x)dx}{\int_{x+0.49\varepsilon}^{\mean^*} f(x) dx}dx}{\int_{p'}^{\mean^*} f(x)dx}
+\frac{\int_{\mean^*-0.49\varepsilon}^{\mean^*} f(x)dx}{\int_{p'}^{\mean^*} f(x)dx}\bigg)\Bigg\}.
\end{align*}
\end{lem}
\begin{proof}
Let the number of arms in the input stream having their means in the range $[0,\mean^*-0.99\varepsilon)$ be $k$. For all $1\leq \ell\leq k$, 
let $S_\ell$ be the set of all possible input streams containing exactly $\ell$ out of the $k$ arms above. Let $t_\ell$ denote the sequence of first $\ell$ arms to arrive in the stream and let $t_\ell=\{c_1,c_2,\ldots,c_\ell\} \in S_\ell$. Let  $king^\ell$ denote the arm that is the king after the $\ell$-th arm has been processed by the algorithm. Let $king^\ell = \arm_i$ with probability $q_{i}$, $\forall i\in[\ell]$, where $q_{i}\geq 0$ for all $i \in [\ell]$, $\sum_{i=1}^{\ell}q_{i} = 1$. Let $X$ be a random variable such that $X=i$ if and only if the first arm in the stream which has mean in the range $[\mean^*-0.99\varepsilon,\mean^*]$ is the $i$-th arm in the stream. Let $X=\ell+1$. Let us assume that the king is $\arm_i$ (with mean $\mean_i$) just before the $(\ell+1)$-th arm arrives where $i\in[\ell]$. 
If the arms arriving in the stream at position $\ell+1$ and later have means in the range $[0,\mean_i+0.49\varepsilon]$, then from Lemma \ref{king:notlose} we know that $\arm_i$ continues to be the king with probability at least $(1-\delta/8)$. 

Let $p'$=max$\{\mean^*-0.99\varepsilon,\mean_i+0.49\varepsilon\}$. Clearly $\mean^*-0.99\varepsilon\leq p' < \mean^* - 0.5\varepsilon$. Consider the first time an arm with mean in the range $[p',\mean^*]$ arrives in the stream. Let $T_1$ be the set of arms whose biases lie in the range $[p',\mean^*]$ and $T_2$ be the set of arms whose biases lie in the range $[\mean^*-0.49\varepsilon,\mean^*]$. Let $\mathcal{A}_1$ be the event that the first arm $a_1$ from $T_1$ which arrives in the stream belongs to $T_2$. Let $\mathcal{A}_2$ be the event that the first arm $a_1$ from $T_1$ which arrives in the stream does not belong to $T_2$. If $king^\ell = \arm_i$ and the event $\mathcal{A}_1$ occurs then $a_1$ is returned as the king by our Algorithm at the end with a probability of at least $(1-\delta/8)^3$. This happens because when the arm $a_1$ arrives, the king at that time has mean less than $\mean^*-0.99\varepsilon$ with probability at least $(1-\delta/8)$, and due to Lemma \ref{king:cond1}, $a_1$ becomes the new king with probability at least $(1-\delta/8)$.  Due to Lemma \ref{king:cond2}, $a_1$ continues to remain as king with probability at least $(1-\delta/8)$. Hence if the event $\mathcal{A}_1$ occurs then $a_1$ is returned as the king by our Algorithm at the end with a probability of at least $(1-\delta/8)^3$.

Now assume that instead the event $\mathcal{A}_2$ has occurred. Let $\mathcal{B}_{a_2}$ be the event that the first arm $a_2$ from the set $T_1$ to arrive in the stream has a mean $\mean_{a_2}$. Let us assume that $\mathcal{B}_{a_2}$ has occurred and $\mean_{a_2}$ belongs to the range $[p'+0.01\varepsilon,\mean^*-0.5\varepsilon)$. Due to Lemma \ref{king:cond1}, $a_2$ becomes the king with probability at least $(1-\delta/8)$. If the means of the arms coming to stream after $a_2$ belongs to the range $[0,\mean_{a_2}+0.49\varepsilon]$, then $a_2$ continues to be the king with probability at least $(1-\delta/8)$. If the first arm $a_3$ arriving in the stream with mean in the range $[\mean_{a_2}+0.5\varepsilon,\mean^*]$ comes before the first arm $a_4$ arriving in the stream with bias in the range $(\mean_{a_2}+0.49\varepsilon,\mean_{a_2}+0.5\varepsilon]$, then $a_3$ becomes the king with probability at least $(1-\delta/8)$ and continues to remain as the king with probability at least $(1-\delta/8)$. Let us denote this event of $a_3$ coming before $a_4$ by $\mathcal{B}_1$.\\
Note that if $X=1$, then we can repeat the above analysis by considering $p' = \mean^*-0.99\varepsilon$. Let $\mathcal{C}_1$ be the event that $\varepsilon$-best arm is returned by the Algorithm and let $\mathcal{C}_2$ be the event that an arm with mean in the range $[\mean^*-0.49\varepsilon,\mean^*]$ is returned by the algorithm.

\begin{align*}
&\mathbb{P}[\mathcal{C}_1|X=\ell+1,king^\ell=c_i,t_\ell=\{c_1,\ldots,c_\ell\}]\\
&\hspace{1cm}\geq  \mathbb{P}[\mathcal{C}_2|\mathcal{A}_1,X=\ell+1,king^\ell=c_i,t_\ell=\{c_1,\ldots,c_\ell\}]\cdot \mathbb{P}[\mathcal{A}_1|X=\ell+1,king^\ell=c_i,t_\ell=\{c_1,\ldots,c_\ell\}]\\
&\hspace{1.5cm}+ \sum_{a_2:p_{a_2}\in [p'+0.01\varepsilon,p-0.5\varepsilon)}\Big(\mathbb{P}[\mathcal{C}_2|\mathcal{B}_1,\mathcal{B}_{a_2},X=\ell+1,king^\ell=c_i,t_\ell=\{c_1,\ldots,c_\ell\}]\\
&\hspace{1.5cm}\cdot \mathbb{P}[\mathcal{B}_1|\mathcal{B}_{a_2},X=\ell+1,king^\ell=c_i,t_\ell=\{c_1,\ldots,c_\ell\}] \cdot \mathbb{P}[\mathcal{B}_{a_2}|X=\ell+1,king^\ell=c_i,t_\ell=\{c_1,\ldots,c_\ell\}]\Big)\\
&\hspace{1cm} \gtrapprox (1-\delta/8)^3\cdot \frac{\int_{p-0.49\varepsilon}^{p} f(x)dx}{\int_{p'}^{p} f(x)dx}
+\int_{\min\{p'+0.01\varepsilon,p-0.5\varepsilon\}}^{p-0.5\varepsilon} (1-\delta/8)^5 \cdot \Bigg(\frac{\int_{x+0.5\varepsilon}^{p} f(x)dx}{\int_{x+0.49\varepsilon}^{p} f(x) dx}\cdot \frac{f(x)dx}{\int_{p'}^{p} f(x)dx}\Bigg)\\
&\hspace{1cm}= (1-\delta)\cdot\Bigg(\frac{\int_{\min\{p'+0.01\varepsilon,p-0.5\varepsilon\}}^{p-0.5\varepsilon} f(x) \frac{\int_{x+0.5\varepsilon}^{p} f(x)dx}{\int_{x+0.49\varepsilon}^{p} f(x) dx}dx}{\int_{p'}^{p} f(x)dx}+\frac{\int_{p-0.49\varepsilon}^{p} f(x)dx}{\int_{p'}^{p} f(x)dx}\Bigg).\\
\end{align*}
Similarly, if $X=1$ then we have the following:
\begin{align*}
&\mathbb{P}[\mathcal{C}_1|X=1]
\gtrapprox (1-\delta)\cdot\Bigg(\frac{\int_{p-0.98\varepsilon}^{p-0.5\varepsilon} f(x) \frac{\int_{x+0.5\varepsilon}^{p} f(x)dx}{\int_{x+0.49\varepsilon}^{p} f(x) dx}dx}{\int_{p-0.99\varepsilon}^{p} f(x)dx}
+\frac{\int_{p-0.49\varepsilon}^{p} f(x)dx}{\int_{p-0.99\varepsilon}^{p} f(x)dx}\Bigg).\\
\end{align*}
\begin{align*}
&\mathbb{P}[\mathcal{C}_1|X=\ell+1,t_\ell =\{c_1,\ldots,c_\ell\}]\\ 
&\hspace{1cm}= \sum_{i=1}^{\ell}\mathbb{P}[king^\ell=c_i|X=\ell+1,t_\ell=\{c_1,\ldots,c_\ell\}]\cdot \mathbb{P}[\mathcal{C}_1|X=\ell+1,king^\ell=c_i,t_\ell=\{c_1,\ldots,c_\ell\}]\\
&\hspace{1cm}=\sum_{i=1}^{\ell}q_{c_i}\cdot \mathbb{P}[\mathcal{C}_1|X=\ell+1,king^\ell=c_i,t_\ell=\{c_1,\ldots,c_\ell\}]\\
&\hspace{1cm}\geq \min_{i\in[\ell]} \mathbb{P}[\mathcal{C}_1|X=\ell+1,king^\ell=c_i,t_\ell=\{c_1,\ldots,c_\ell\}].\\
\end{align*}
We have, 
\begin{align*}
\mathbb{P}[\mathcal{C}_1|X=\ell+1] &= \sum_{(c_i)_{i\in[\ell]} \thicksim \mathcal{S}_\ell}\mathbb{P}[t_\ell=\{c_1,\ldots,c_\ell\}|X=\ell+1]\cdot \mathbb{P}[\mathcal{C}_1|X=\ell+1,t_\ell=\{c_1,\ldots,c_\ell\}]\\
&\geq \min_{(c_i)_{i\in[\ell]}\thicksim \mathcal{S}_\ell}\mathbb{P}[\mathcal{C}_1|X=\ell+1,t_\ell=\{c_1,\ldots,c_\ell\}].\\
\end{align*}
\begin{align*}
\mathbb{P}[\mathcal{C}_1] &= \sum _{i\in[k+1]}\mathbb{P}[X=i]\cdot \mathbb{P}[\mathcal{C}_1|X=i]\\
& \geq \min_{i\in[k+1]}\mathbb{P}[\mathcal{C}_1|X=i]\\
 &\geq \min\Big\{\mathbb{P}[\mathcal{C}_1|X=1],
 \min_{(c_i)_{i\in[\ell]}\thicksim \mathcal{S}_\ell}\min_{i\in[\ell]} \mathbb{P}[\mathcal{C}_1|X=\ell+1,king^\ell=c_i,t_\ell=\{c_1,\ldots,c_\ell\}]\Big\}\\
&\gtrapprox \inf_{p'\in[p-0.99\varepsilon,p-0.5\varepsilon]}\Bigg\{(1-\delta)\cdot\Bigg(\frac{\int_{\min\{p'+0.01\varepsilon,p-0.5\varepsilon\}}^{p-0.5\varepsilon} f(x) \frac{\int_{x+0.5\varepsilon}^{p} f(x)dx}{\int_{x+0.49\varepsilon}^{p} f(x) dx}dx}{\int_{p'}^{p} f(x)dx}+\frac{\int_{p-0.49\varepsilon}^{p} f(x)dx}{\int_{p'}^{p} f(x)dx}\Bigg)\Bigg\}. 
\end{align*}
\end{proof}
\subsection{Performance under various distributions}
The following distributions are truncated distributions and the support is (0,1]. Note that the following calculations are made assuming $\varepsilon=\frac{1}{10}$. Lower bound on the probability that Algorithm \ref{algo:new_algo} returns an $\varepsilon$-best arm, for various truncated distributions like Normal, lognormal, exponential, beta, gamma, Weibull, and uniform distribution is at least $0.9(1-\delta)$. Note that as $\varepsilon$ tends to 0, the distributions mentioned earlier behave similar to uniform distribution on the range $[1-\varepsilon,1]$. So the lower bound of $\mathbb{P}$[Algorithm \ref{algo:new_algo} returns an $\varepsilon$-best arm] tends to $0.927(1-\delta)$ which is the lower bound on this probability for the uniform distribution and it does not change with $\varepsilon$.

\begin{table}[ht!]
	\begin{center} 
		\begin{tabular}{ |c|c|c| }
			\hline
			\multirow{2}{*}{Distribution}& \multirow{2}{*}{$f(x)$} & Probability that Algo. 4\\
			& &returns an $\varepsilon$-best arm \\ 
			\hline
			Uniform & 1 & $\geq 0.927(1-\delta)$\\ 
			\hline
			Normal & $\frac{e^{-x^2}}{\sqrt{\pi}}$ & $\geq 0.924(1-\delta)$ \\ 
			\hline
			Gamma & $\frac{x^{-0.5}}{\Gamma(0.5)}\cdot e^{-x}$&$\geq 0.924(1-\delta)$\\
			\hline
			Beta & $\frac{x^{9}\cdot\Gamma(11)}{\Gamma(10)}$&$\geq 0.942(1-\delta)$\\
			\hline
			Exponential & $2e^{-2x}$ &$\geq 0.923(1-\delta)$\\
			\hline
			\multirow{2}{*}{Weibull}&$x<0\ :\ 0$&\\
			&$x \geq 0:\ 2xe^{-x^{2}}$&{$\geq 0.925(1-\delta)$}\\
			\hline
			\multirow{2}{*}{Lognormal}&$\frac{1}{x\cdot\sqrt{3\pi}}\cdot$&\\
			&$e^{-(\ln(x))^2/3}$&{$\geq 0.925(1-\delta)$}\\
			\hline
			
		\end{tabular}
	\end{center}
\end{table}

\section{Experiments}\label{exp:dist123}
We now 
provide detailed experimental evaluation of Algorithm \ref{algo:new_algo} and show that it returns an $\varepsilon$-best arm with high confidence even when we reduce the number of samples $(s_\ell)$ per arm at each level $\ell$ by a factor of 40000.
We ran the algorithm on $R=100$ different instances. For each instance, the means of each of the $n=10^5$ arms were sampled from a distribution $\mathcal{D}$ with support $(0,1]$, mean $=\mu$, and variance $=\sigma^2$. Note that if $\mathcal{D}$ is a truncated distribution then $\mu$ and $\sigma^2$ denote the mean and variance of the non-truncated version of $\mathcal{D}$, denoted $\mathcal{D}^\prime$. Also we set $C=117, \varepsilon={1}/{10}, \delta={1}/{10}$. For $i \in [n]$, let $\mean_i$ be the mean obtained for $\arm_i$. The reward distribution of $\arm_i$ is then Bernoulli$(\mean_i)$. 

We next provide the details of the distributions corresponding to the figures in Figure \ref{fig:experiments in appendix}. Note that we consider the truncated version, $\mathcal{D}$, of the following distributions, $\mathcal{D}^\prime$, supported on $(0,1]$.
\begin{enumerate}
	\item Figure (\ref{fig2a}): $\mathcal{D}^\prime = \text{Beta}(\alpha = 10, \beta = 1)$
	\item Figure (\ref{fig2b}): $\mathcal{D}^\prime = \text{Exp}(\lambda = 2)$ 
	\item Figure (\ref{fig2c}): $\mathcal{D}^\prime = \text{Gamma}(k = 0.5,\theta = 1)$
	\item Figure (\ref{fig2d}): $\mathcal{D}^\prime = \mathcal{N}(\mu=1/2,\sigma^2=1/10)$
	\item Figure (\ref{fig2e}): $\mathcal{D}^\prime = \mathcal{N}(\mu=1/2,\sigma^2=1/5)$
	\item Figure (\ref{fig2f}): $\mathcal{D}^\prime = \mathcal{N}(\mu=1/2,\sigma^2=1/100)$
	\item Figure (\ref{fig2g}): $\mathcal{D}^\prime = \mathcal{N}(\mu=1/2,\sigma^2=1/20)$
	\item Figure (\ref{fig2h}): $\mathcal{D}^\prime = \text{lognormal}(\mu=0,\sigma^2=3/2)$
	\item Figure (\ref{fig2i}): $\mathcal{D}^\prime = \text{Weibull}(\lambda = 1, k = 2)$
\end{enumerate}
In all these cases, we almost always return an arm with mean within at most 0.05 ($<1/10=\varepsilon$) from the mean of the best-arm.
\begin{figure}
	\centering
	\begin{subfigure}[b]{0.30\columnwidth}
		\centering
		\includegraphics[scale=0.3]{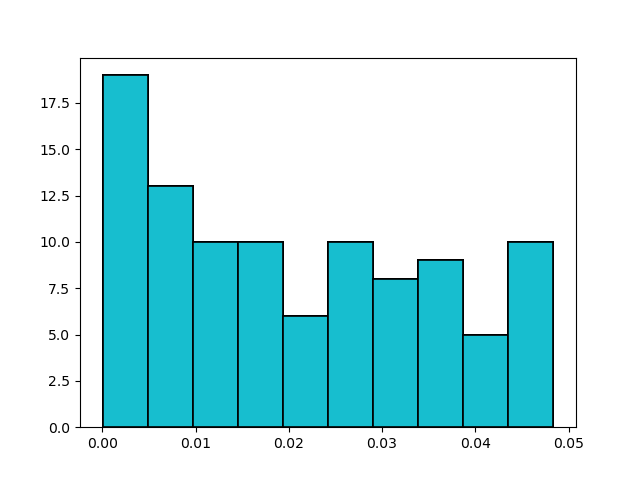}
		\caption[RandomInstance]%
		{{\small Instance 1}}    
		\label{fig2a}
	\end{subfigure}
	\hfill
	\begin{subfigure}[b]{0.30\columnwidth}  
		\centering 
		\includegraphics[scale=0.3]{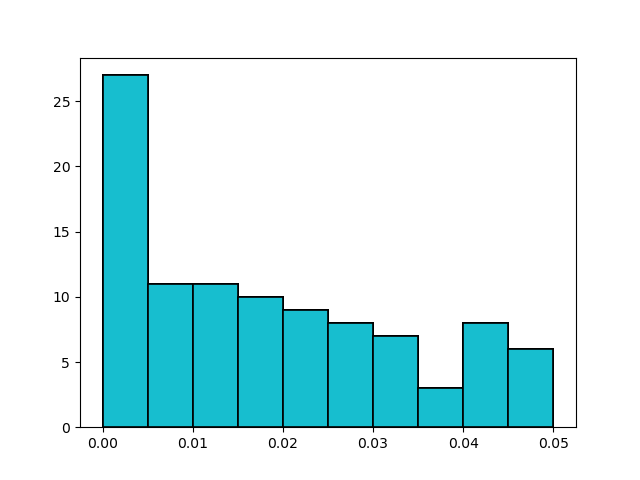}
		\caption[]%
		{{\small Instance 2}}    
		\label{fig2b}
	\end{subfigure}
	\hfill
	\begin{subfigure}[b]{0.30\columnwidth}   
		\centering 
		\includegraphics[scale=0.3]{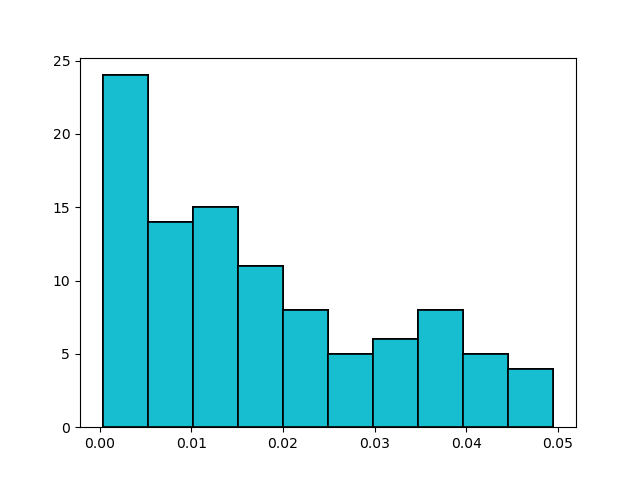}
		\caption[]%
		{{\small Instance 3}}    
		\label{fig2c}
	\end{subfigure}
	\hfill
	\begin{subfigure}[b]{0.30\columnwidth}   
		\centering 
		\includegraphics[scale=0.3]{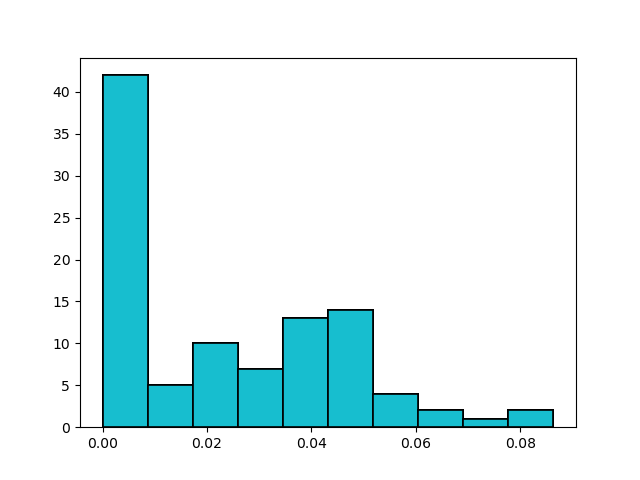}
		\caption[]%
		{{\small Instance 4}}    
		\label{fig2d}
	\end{subfigure}
	\hfill
	\begin{subfigure}[b]{0.30\columnwidth}   
		\centering 
		\includegraphics[scale=0.3]{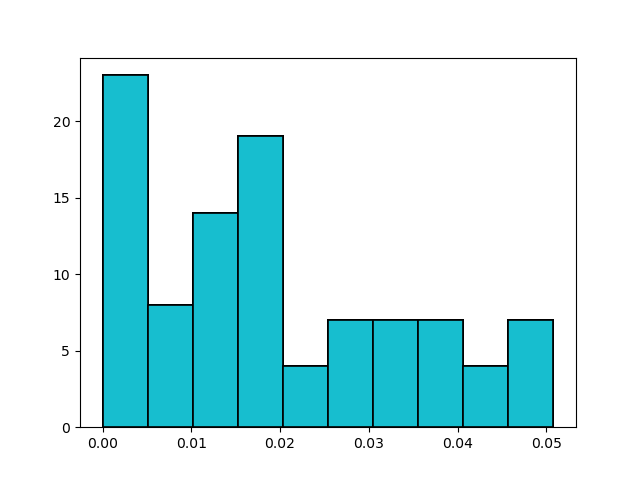}
		\caption[]%
		{{\small Instance 5}}    
		\label{fig2e}
	\end{subfigure}
	\hfill
	\begin{subfigure}[b]{0.30\columnwidth}   
		\centering 
		\includegraphics[scale=0.3]{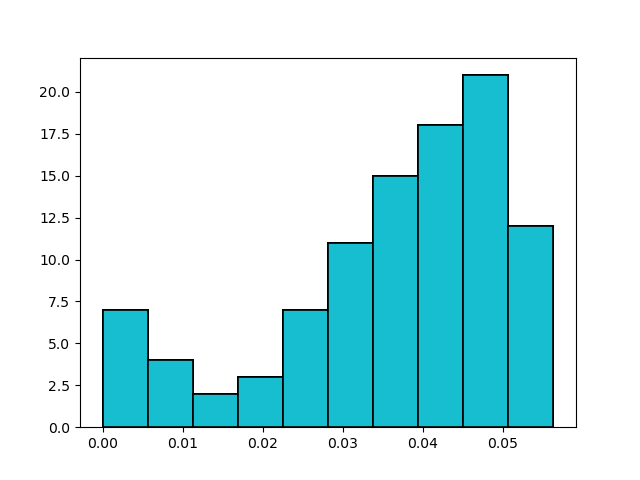}
		\caption[]%
		{{\small Instance 6}}    
		\label{fig2f}
	\end{subfigure}
	\hfill
	\begin{subfigure}[b]{0.30\columnwidth}   
		\centering 
		\includegraphics[scale=0.3]{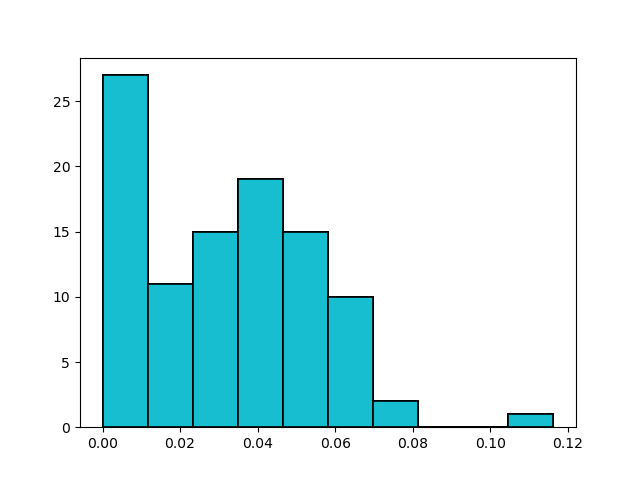}
		\caption[]%
		{{\small Instance 7}}    
		\label{fig2g}
	\end{subfigure}
	\hfill
	\begin{subfigure}[b]{0.30\columnwidth}   
		\centering 
		\includegraphics[scale=0.3]{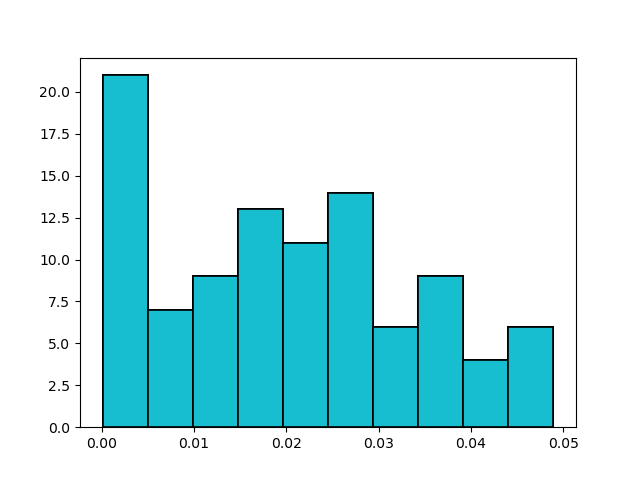}
		\caption[]%
		{{\small Instance 8}}    
		\label{fig2h}
	\end{subfigure}
	\hfill
	\begin{subfigure}[b]{0.30\columnwidth}   
		\centering 
		\includegraphics[scale=0.3]{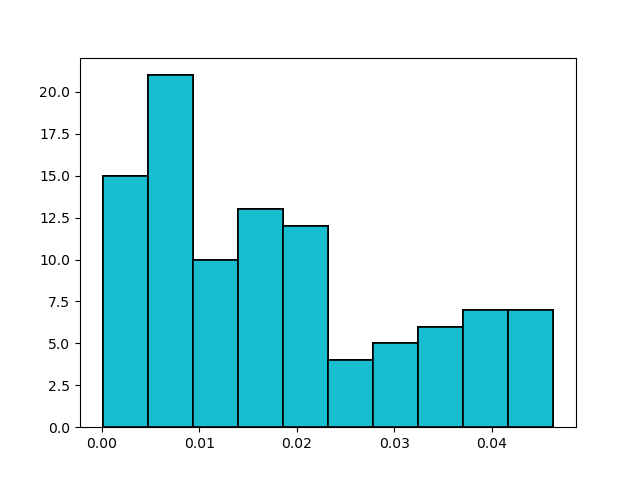}
		\caption[]%
		{{\small Instance 9}}    
		\label{fig2i}
	\end{subfigure}
	
	\caption[]
	{\small \textbf{X-axis}: Gap between the means of the best arm and arm returned by Algorithm \ref{algo:new_algo}, \textbf{Y-axis}: Count of such arms.} 
	\label{fig:experiments in appendix}
\end{figure}
\end{document}